\documentclass[twoside]{article}

\usepackage[accepted]{aistats2022}

\usepackage[backend=biber,doi=false,url=false,style=apa,natbib=true]{biblatex}
\usepackage{import}
\usepackage{mathtools}
\usepackage{amssymb}
\usepackage{mathrsfs}
\usepackage{amsthm}
\usepackage{xparse}
\usepackage{import}
\usepackage[normalem]{ulem}

\newcommand{\scrbar}[1]{\overline{\mathcal{#1}}}

\ExplSyntaxOn

\cs_new_protected:Nn \bamboo_define:nnnnnN
{
	\cs_new_protected:cpx { #3 #1 #4 } { \exp_not:N #5{#6{#2}} }
}

\int_step_inline:nnn { `A } { `Z }
{
	\bamboo_define:nnnnnN
	{ \char_generate:nn { #1 } { 11 } }
	{ \char_generate:nn { #1 } { 11 } }
	{ bl }
	{ }
	{ \mathbf }
	\use:n
	
	\bamboo_define:nnnnnN
	{ \char_generate:nn { #1 } { 11 } }
	{ \char_generate:nn { #1 } { 11 } }
	{ scr }
	{ }
	{ \mathcal }
	\use:n
	
	\bamboo_define:nnnnnN
	{ \char_generate:nn { #1 } { 11 } }
	{ \char_generate:nn { #1 } { 11 } }
	{ }
	{ hat }
	{ \widehat }
	\use:n
	
	\bamboo_define:nnnnnN
	{ \char_generate:nn { #1 } { 11 } }
	{ \char_generate:nn { #1 } { 11 } }
	{ }
	{ br }
	{ \overline }
	\use:n
	
	\bamboo_define:nnnnnN
	{ \char_generate:nn { #1 } { 11 } }
	{ \char_generate:nn { #1 } { 11 } }
	{ }
	{ tl }
	{ \widetilde }
	\use:n
	
	\bamboo_define:nnnnnN
	{ \char_generate:nn { #1 } { 11 } }
	{ \char_generate:nn { #1 } { 11 } }
	{ scr }
	{ br }
	{ \scrbar }
	\use:n
}
\int_step_inline:nnn { `a } { `z }
{
	\bamboo_define:nnnnnN
	{ \char_generate:nn { #1 } { 11 } }
	{ \char_generate:nn { #1 } { 11 } }
	{ bl }
	{ }
	{ \mathbf }
	\use:n
	
	\bamboo_define:nnnnnN
	{ \char_generate:nn { #1 } { 11 } }
	{ \char_generate:nn { #1 } { 11 } }
	{ }
	{ hat }
	{ \hat }
	\use:n
	
	\bamboo_define:nnnnnN
	{ \char_generate:nn { #1 } { 11 } }
	{ \char_generate:nn { #1 } { 11 } }
	{ }
	{ br }
	{ \bar }
	\use:n
	
	\bamboo_define:nnnnnN
	{ \char_generate:nn { #1 } { 11 } }
	{ \char_generate:nn { #1 } { 11 } }
	{ }
	{ tl }
	{ \tilde }
	\use:n                            
}
\clist_map_inline:nn
{
	Gamma,Delta,Theta,Lambda,Xi,Pi,Sigma,Phi,Psi,Omega
}
{
	\bamboo_define:nnnnnN
	{ #1 }
	{ #1 }
	{ bl }
	{ }
	{ \mathbf }
	\use:c
	
	\bamboo_define:nnnnnN
	{ #1 }
	{ #1 }
	{ op }
	{ }
	{ \op }
	\use:c
	
	\bamboo_define:nnnnnN
	{ #1 }
	{ #1 }
	{ scr }
	{ }
	{ \mathcal }
	\use:c
	
	\bamboo_define:nnnnnN
	{ #1 }
	{ #1 }
	{ }
	{ hat }
	{ \widehat }
	\use:c
	
	\bamboo_define:nnnnnN
	{ #1 }
	{ #1 }
	{ }
	{ br }
	{ \overline }
	\use:c
	
	\bamboo_define:nnnnnN
	{ #1 }
	{ #1 }
	{ }
	{ tl }
	{ \widetilde }
	\use:c
}
\clist_map_inline:nn
{
	alpha,beta,gamma,delta,epsilon,zeta,eta,theta,iota,kappa,
	lambda,mu,nu,xi,pi,rho,sigma,tau,phi,chi,psi,omega
}
{
	\bamboo_define:nnnnnN
	{ #1 }
	{ #1 }
	{ bl }
	{ }
	{ \mathbf }
	\use:c
	
	\bamboo_define:nnnnnN
	{ #1 }
	{ #1 }
	{ }
	{ hat }
	{ \hat }
	\use:c
	
	\bamboo_define:nnnnnN
	{ #1 }
	{ #1 }
	{ }
	{ br }
	{ \bar }
	\use:c
	
	\bamboo_define:nnnnnN
	{ #1 }
	{ #1 }
	{ }
	{ tl }
	{ \tilde }
	\use:c
}

\ExplSyntaxOff

\DeclareMathOperator{\E}{\mathbf{E}}
\renewcommand{\P}{\operatorname{\mathbf{P}}}

\newcommand{\tr}{\operatorname{tr}}
\newcommand{\argmin}{\operatornamewithlimits{arg~min}}

\newcommand{\diag}{\operatorname{diag}}

\DeclarePairedDelimiter{\norm}{\lVert}{\rVert}

\DeclarePairedDelimiter{\abs}{\lvert}{\rvert}

\DeclarePairedDelimiter{\braces}{\{}{\}}
\DeclarePairedDelimiter{\parens}{(}{)}
\DeclarePairedDelimiter{\brackets}{[}{]}
\DeclarePairedDelimiterX{\ip}[2]{\langle}{\rangle}{#1,#2}

\DeclarePairedDelimiterXPP{\normsub}[2]{}{\lVert}{\rVert}{_{#2}}{#1}
\DeclarePairedDelimiterXPP{\ipsub}[3]{}{\langle}{\rangle}{_{#3}}{#1,#2}

\DeclarePairedDelimiterXPP{\ipHS}[2]{}{\langle}{\rangle}{_{\mathrm{HS}}}{#1, #2}
\DeclarePairedDelimiterXPP{\normHS}[1]{}{\lVert}{\rVert}{_{\mathrm{HS}}}{#1}

\DeclarePairedDelimiterXPP{\ipF}[2]{}{\langle}{\rangle}{_{\mathrm{F}}}{#1, #2}
\DeclarePairedDelimiterXPP{\normF}[1]{}{\lVert}{\rVert}{_{\mathrm{F}}}{#1}

\DeclarePairedDelimiterXPP{\dkl}[2]{\operatorname{D_{KL}}}{(}{)}{}{#1 \: \delimsize\Vert \: #2}

\DeclarePairedDelimiterXPP{\restr}[2]{}{{}}{\vert}{_{#2}}{#1}

\newcommand{\R}{\mathbf{R}}

\newcommand{\var}{\operatorname{var}}

\newcommand{\indicator}[1]{\mathbf{1}_{\{ #1 \}}}

\newcommand{\spn}{\operatorname{span}}

\newcommand{\given}{\:\vert\:}
\newcommand{\sign}{\operatorname{sign}}
\newcommand{\simiid}{\overset{\mathclap{\text{i.i.d.}}}{\sim}}

\usepackage{array}
\usepackage{booktabs}
\usepackage{suffix}
\usepackage{bookmark}
\usepackage{color}
\usepackage{graphicx}
\graphicspath{{./figures/}}
\usepackage{subcaption}
\usepackage[capitalize]{cleveref}
\usepackage{authblk}
\usepackage{ulem}

\usepackage{hyperref}
\hypersetup{
	bookmarks=true,
	colorlinks,
	linkcolor=blue,
	citecolor=blue,
	urlcolor=blue,
}

\newcommand\ipH[2]{\ipsub{#1}{#2}{\scrH}}

\newcommand\normH[1]{\normsub{#1}{\scrH}}

\newtheorem{lemma}{Lemma}
\newtheorem{theorem}{Theorem}
\newtheorem{corollary}{Corollary}

\Crefname{assumption}{Assumption}{Assumptions}

\begin{document}
%
\runningtitle{Harmless interpolation with structured features}

%

 \twocolumn[

 \aistatstitle
{Harmless interpolation in regression and classification\\ with structured features}

 \aistatsauthor{ Andrew D.\ McRae \And Santhosh Karnik \And  Mark A.\ Davenport \And Vidya Muthukumar}

 \aistatsaddress{ Georgia Tech  \And  Michigan State University \And Georgia Tech \And Georgia Tech } ]


\begin{abstract}
	Overparametrized neural networks tend to perfectly fit noisy training data yet generalize well on test data. Inspired by this empirical observation, recent work has sought to understand this phenomenon of \emph{benign overfitting} or \emph{harmless interpolation} in the much simpler linear model. Previous theoretical work critically assumes that either the data features are statistically independent or the input data is high-dimensional; this precludes general nonparametric settings with structured feature maps. In this paper, we present a general and flexible framework for upper bounding regression and classification risk in a reproducing kernel Hilbert space. A key contribution is that our framework describes precise sufficient conditions on the data Gram matrix under which harmless interpolation occurs. Our results recover prior independent-features results (with a much simpler analysis), but they furthermore show that harmless interpolation can occur in more general settings such as features that are a bounded orthonormal system. Furthermore, our results show an asymptotic separation between classification and regression performance in a manner that was previously only shown for Gaussian features. 
\end{abstract}

\begin{figure*}[t]
\centering
\begin{subfigure}[t]{0.32\textwidth}
    \centering
    \includegraphics[width=0.99\textwidth]{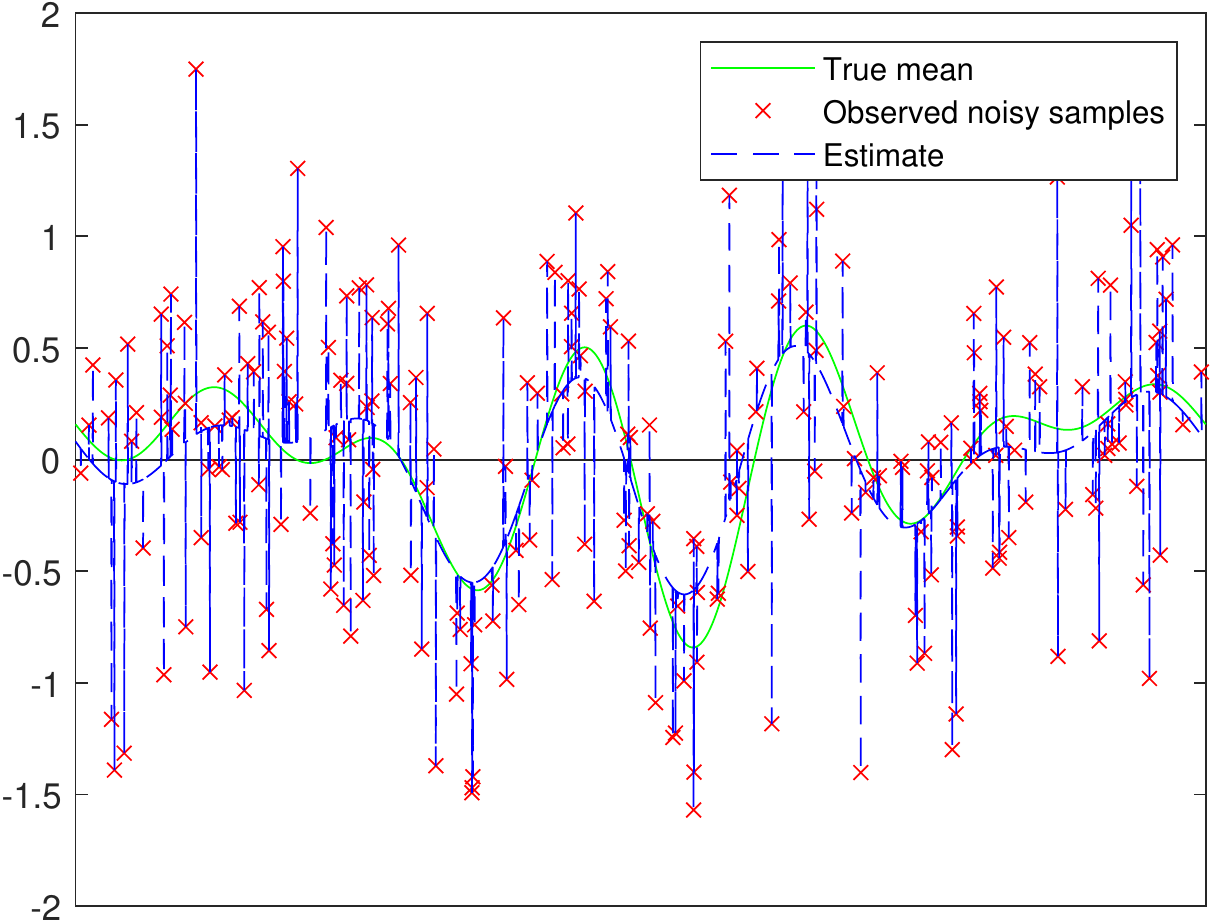}
    \caption{\small Gaussian-noise observations}
\end{subfigure}\hfill
\begin{subfigure}[t]{0.32\textwidth}
    \centering
    \includegraphics[width=0.99\textwidth]{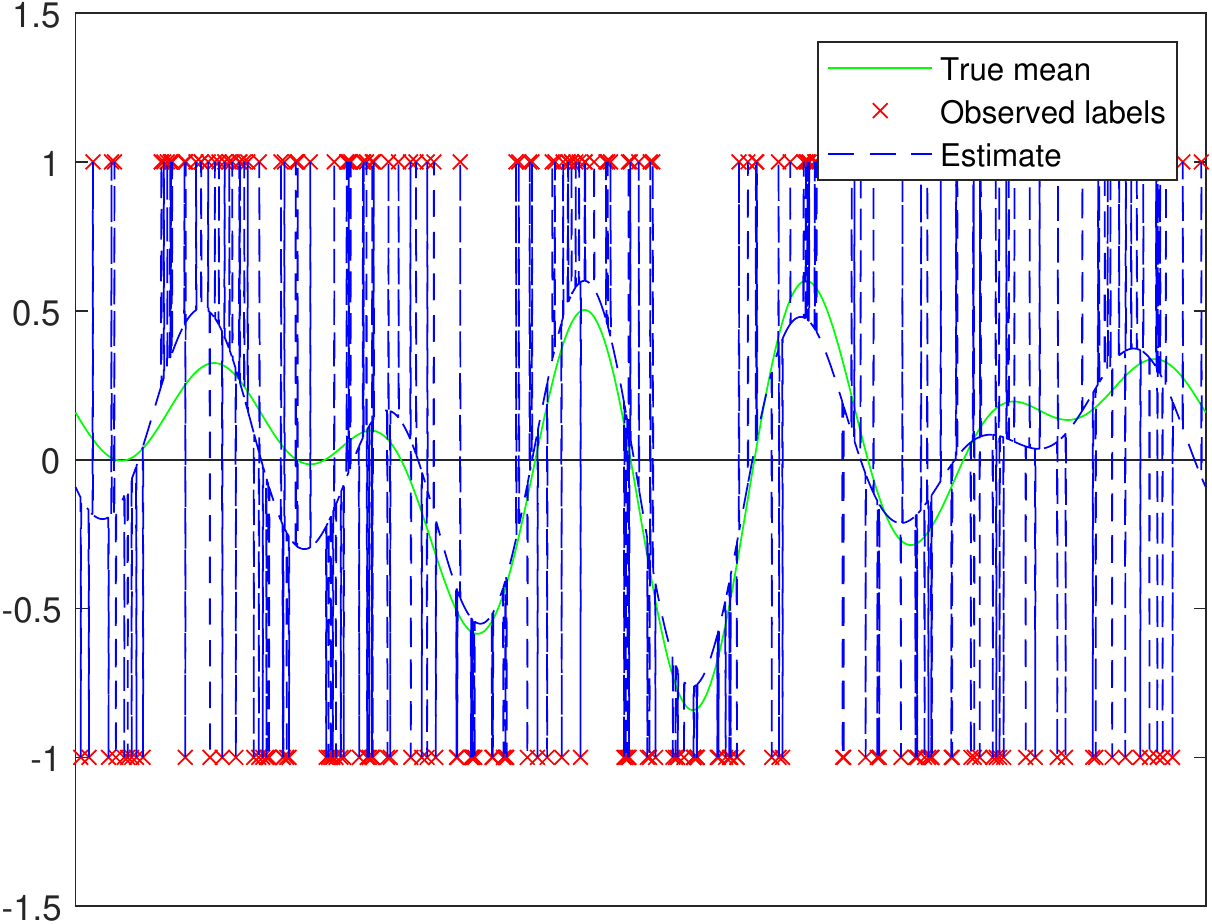}
    \caption{\small Binary labels---good regression performance}
\end{subfigure}\hfill
\begin{subfigure}[t]{0.32\textwidth}
    \centering
    \includegraphics[width=0.99\textwidth]{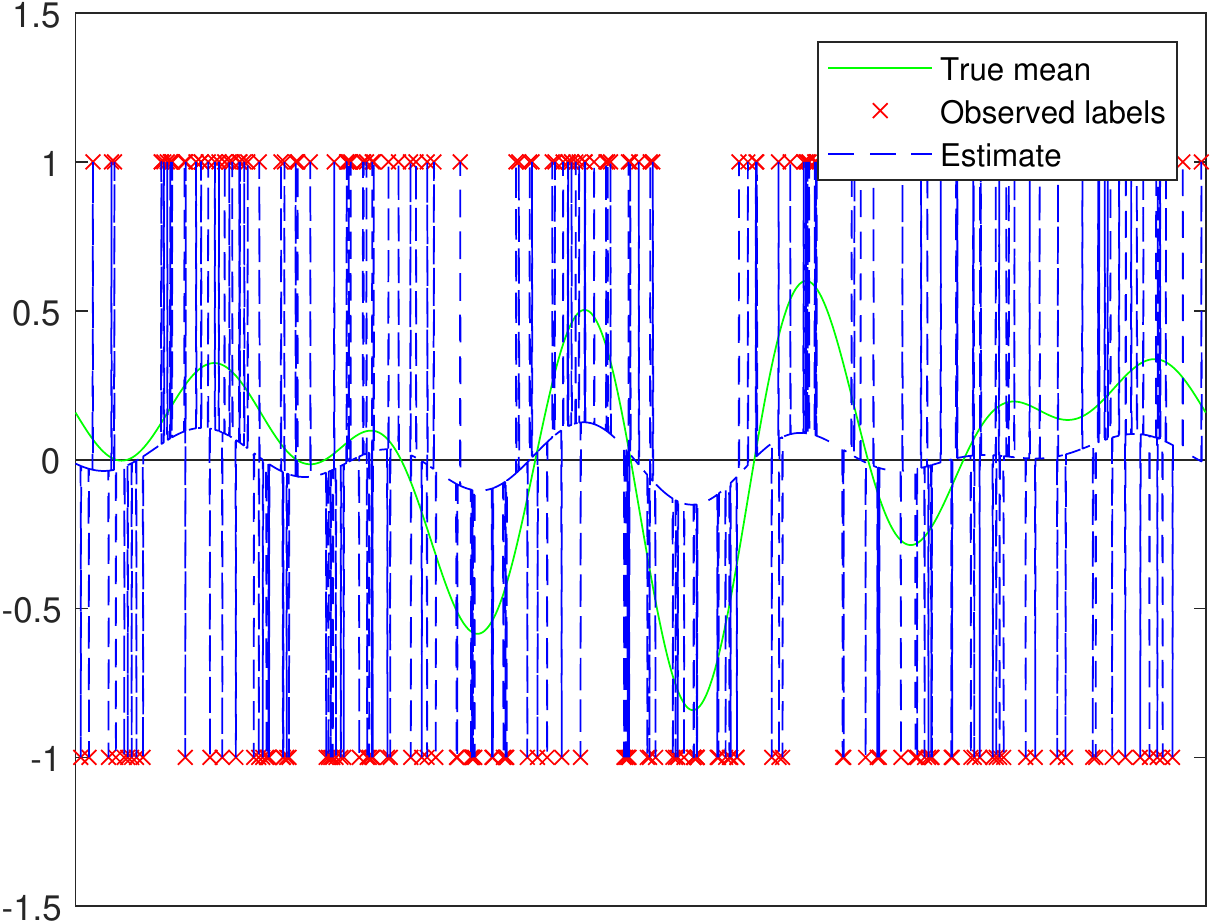}
    \caption{\small Binary labels---poor regression but good classification performance}
\end{subfigure}
\caption{\small Interpolation in various regimes. This uses the bi-level Fourier series framework of \Cref{sec:numerical}.} \vspace{-5mm}
\label{fig:interp_regimes}
\end{figure*}

\section{INTRODUCTION}
Overparametrized neural networks tend to perfectly fit, or \emph{interpolate}, noisy training data. 
Somewhat surprisingly, these overparametrized networks also tend to generalize well \citep{Zhang2017}. More recently, this phenomenon of ``harmless interpolation" was also empirically demonstrated in the much simpler model families of kernel machines \citep{Belkin2018} and overparametrized linear models \citep{Belkin2019}. These observations have motivated a large body of research that aims to develop a mathematical understanding of the generalization properties of interpolating solutions and the impact of fitting noise (see \Cref{sec:relwork,app:litreview} for more related work).

While these theoretical results 
represent significant progress, they come with some caveats.
Most notably, harmless interpolation has only been shown under (a) strong assumptions on the feature distribution or (b) high dimension of the input data.
For example, the strongest guarantees on harmless interpolation assume that the features consist of independent random variables (or are a linear transformation of such a vector).
Similarly, the positive results on consistency of kernel interpolation require the dimension of the input data to grow with the size of the training set.

To see why these assumptions may not be realistic, consider the problem of simple linear regression using a Fourier series model $f(x) = \sum_{\ell} a_\ell e^{i 2\pi \ell x}$ for a function $f$ on the interval $[0, 1]$,
where $\ell$ may range over all integers or, as we will later assume, a subset $\{-d, \dots, d\}$.
Here the input data dimension is $1$, and the features are given by $v_{\ell}(x) = e^{\mathbf{j} 2\pi \ell x}$. 
If $x$ is uniformly distributed,
the features $\{v_{\ell}(x)\}_\ell$ (all evaluated at the same random $x$), though \emph{uncorrelated}, are not \emph{independent}.
In this (and many other) examples, the input data can be low-dimensional and the features may not be independent.
Whether harmless interpolation is possible with high-dimensional feature maps on such \emph{constant}-dimensional data remains an open question.
As a first effort, \cite{Muthukumar2020a} show that harmless interpolation can occur with structured feature maps with uniformly spaced data,
but whether this can be shown for the more realistic case of randomly-sampled data has remained open.

A second question is how the interpolation phenomenon applies to the \emph{classification} problem.
For example, \cite{Muthukumar2021,Chatterji2021}
show that the max-margin support vector machine can achieve good performance even when the corresponding regression task does not.
These results require the very strong assumption of independent (sub)Gaussian features.
Whether this asymptotic separation between regression and classification tasks exists in more general kernel settings is not addressed by this literature.

\subsection{Our Contributions}
In this paper, we provide new non-asymptotic risk bounds for both regression and classification tasks with the standard Hilbert-norm regularizer under minimal regularity assumptions. 
Our results apply for an arbitrarily small amount of regularization (including the interpolating regime) and are summarized below.

\textbf{Harmless interpolation in kernel regression.}
For the regression task, we obtain new non-asymptotic risk bounds on the mean-squared-error of the Hilbert-norm regularized estimator, which includes the cases of kernel ridge regression and minimum-Hilbert-norm interpolation.
In \Cref{sec:results_determ}, we give error bounds for fixed sample locations.
In \Cref{sec:operator_conc}, we give a variety of concentration results from random sampling that, when combined with our fixed-sample theorems, yield high-probability guarantees of harmless interpolation.
Our results imply harmless interpolation in significantly more general settings than previous works
(see~\Cref{sec:relwork} for a comparison to prior work).
Our results recover existing independent-feature results (e.g., \cite{Bartlett2020}) but also apply to other examples such as bounded orthonormal systems (BOSs).
BOSs include many popular feature ensembles such as sinusoids and Chebyshev polynomials.
\Cref{fig:interp_regimes}a shows an example of a function estimate that yields strong regression performance for the case of sinusoidal Fourier basis features.

\textbf{Asymptotic separation between kernel classification and regression.}
We next analyze the classification error of the minimum-Hilbert-norm interpolator of binary labels.
Although good regression performance implies good classification performance (see \Cref{fig:interp_regimes}b), the reverse is not true.
In \Cref{sec:class}, we derive a simple bound on classification error that can be much tighter than the bound on regression error,
and we present another fixed-sample error bound useful for bounding the regression risk.
Then, for the case of bounded orthonormal system features,
we demonstrate an asymptotic separation between the regression and classification tasks.
\Cref{fig:interp_regimes}c illustrates how the minimum-norm label interpolator can have poor regression performance but good classification performance.

\subsection{Related Work}\label{sec:relwork}
\textbf{Harmless interpolation.}
Recent work has shown that the ``harmless interpolation" phenomenon becomes more pronounced with increased (effective) overparameterization when the minimum-Hilbert-norm interpolator is used in kernel regression or the minimum-norm interpolator is used in linear regression in a variety of models.
All of the current known harmless interpolation results assume at least one of the following (see \Cref{app:litreview} for a more complete citation list): (a) independence of features, (b) sub-Gaussianity in the feature vector, (c) high data dimension, or (d) explicit structure in the kernel operator/feature map.
For specific kernels like the Laplace kernel, statistically consistent interpolation may actually \emph{require} growing data dimension with the number of training examples~\citep{Rakhlin2019}, as the data dimension fundamentally alters the eigenvalues of the Laplace kernel integral operator.
In contrast, our results do not explicitly posit any of these assumptions.
Our sufficient conditions for harmless interpolation are expressed purely in terms of the eigenvalues of the kernel integral operator, and do not utilize special structure either on the eigenfunctions or the integral operator itself.

\textbf{Classification versus regression.} General techniques from statistical learning theory (e.g., \cite{Vapnik2000,Bartlett2002}) do not differentiate between classification and regression tasks.
However, the idea that classification is easier than regression is well-known: the main idea is we do not need near-zero bias, but rather how much signal is recovered only needs to be large relative to the variance. 
This idea goes back to~\citep{Friedman1997}, and has primarily been used to obtain faster non-asymptotic rates for classification relative to regression in a number of scenarios~\citep{Devroye1996,Koltchinskii2005,Audibert2007}.
A separation in statistical consistency between the two tasks was shown more recently in~\cite{Muthukumar2021}.
Similar sharp analyses for classification error have also been provided for the related high-dimensional linear discriminant analysis setting~\citep{Chatterji2021,Wang2021,Cao2021}.
These results all make restrictive assumptions of Gaussianity, independent sub-Gaussian features, or Gaussian mixture models; the most fine-grained analyses \citep{Muthukumar2021,Wang2021} require Gaussian design.
With our more general analysis, we show that the previous restrictive assumptions can be avoided and demonstrate that the separation between classification and regression consistency is a general phenomenon.
Although our error expressions are less sharp nonasymptotically than those that assume Gaussian features,
the consistency implications are nearly identical.

\textbf{General kernel regression.}
Finally, our work continues a substantial literature on general linear and RKHS regression.
Space limitations prevent a comprehensive review,
but we note that our analysis techniques most closely resemble the approach of \cite{Zhang2005,Hsu2014}, who analyze explicitly regularized ridge regression under random design with minimal assumptions on the data distribution.
Other notable works are \cite{Caponetto2007,Steinwart2009},
which also use techniques based on the kernel integral operator. These works assume a power-law eigenvalue decay to get power-law regression error bounds.
Our results apply to more general kernels with an arbitrary eigenvalue decay and give a more refined bias-variance decomposition of error.
Significantly, none of these works analyze interpolating solutions in the presence of noise.

\section{KERNEL REGRESSION}
Our results are presented in terms of reproducing kernel Hilbert space regression (with traditional linear regression as a special case).
We first introduce the analytical framework and then present our main results.
\subsection{Kernel Regression Introduction}
We first review the general theory of regression in reproducing kernel Hilbert spaces.
A more thorough introduction to kernel theory can be found in many standard references, such as \cite{Schoelkopf2002}, \cite{Wendland2004}, and Chapter 12 of \cite{Wainwright2019}.

Let $X$ be a set, and let $\scrH$ be a real reproducing kernel Hilbert space over $X$ with kernel $k \colon X \times X \to \R$.
For $f \in \scrH$ and $x \in X$,
we have $f(x) = \ipH{f}{k_x}$,
where $k_x \coloneqq k(\cdot, x)$.
Note that this implies that $k(x, y) = \ipH{k_x}{k_y}$.

Suppose $f^* \in \scrH$, and we observe $y_i = f^*(x_i) + \xi_i$, $i = 1,\dots, n$,
where $x_1, \dots, x_n \in X$ are sample points, and the $\xi_i$'s represent noise or other measurement error.
We use the kernel ridge regression estimate
\[
\fhat = \argmin_{f \in \scrH}~\sum_{i=1}^n (y_i - f(x_i))^2 + \alpha \normH{f}^2,
\]
where $\alpha \geq 0$ is a regularization term.
When $\alpha \to 0$, we get the minimum-Hilbert-norm interpolator,
\[
	\fhat = \argmin_{f \in \scrH}~\normH{f}~\text{s.t.}~f(x_i) = y_i~\forall i = 1,\ldots,n.
\]
By the standard kernel regression formula, we have $\fhat(x) = \sum_{i = 1}^n \zhat_i k(x, x_i)$ where the vector $\zhat = \parens*{ \alpha I_n + K}^{-1} y$,
and $K$ is the kernel Gram matrix with $K_{ij} = \ipH{k_{x_i}}{k_{x_j}} = k(x_i, x_j)$.
We denote by $\scrA\colon \scrH \to \R^n$ the \emph{sampling operator}, which is defined by $(\scrA(f))_i = f(x_i) = \ipH{f}{k_{x_i}}$.
The adjoint of the sampling operator is given by $\scrA^*(z) = \sum_{i=1}^n z_i k_{x_i}$ for all $z \in \R^n$.
Then the Gram matrix is $K = \scrA \scrA^*$, and we can write the kernel regression estimate in terms of the standard ridge regression formulas:
\[
	\fhat = \scrA^* (\alpha I_n + \scrA \scrA^*)^{-1} y = (\alpha \scrI_{\scrH} + \scrA^* \scrA)^{-1} \scrA^* y.
\]
Note that, in general, the second expression is only well-defined if $\alpha > 0$ (since $\scrA$ is rank-deficient if $\scrH$ is infinite-dimensional).

We analyze two terms in this estimator.
The first is the estimator that would be obtained in the absence of noise, which is given by
\begin{align*}
	\fhat_0 &\coloneqq \scrA^* (\alpha I_n + \scrA \scrA^*)^{-1} \scrA f^* = (\alpha \scrI_{\scrH} + \scrA^* \scrA)^{-1} \scrA^* \scrA f^*.
\end{align*}
The second is the contribution to the estimate due to noise,
which we denote by the function $\epsilon(x)$. We have
\[
\epsilon = \scrA^* (\alpha I_n + \scrA \scrA^*)^{-1} \xi,
\]
where $\xi = (\xi_1, \dots, \xi_n)$.
This leads to a standard decomposition in the error of the estimator $\fhat$ in terms of its bias and variance.

To characterize the test error, we need a sampling model.
Let $\mu$ be a probability measure on $X$.
We then define the kernel integral operator $\scrT$ as
\[
	(\scrT(f))(x) = \int_X k(x, y) f(y)~d\mu(y)
\]
with respect to the measure $\mu$.
Under mild regularity/continuity conditions (see, e.g., \cite{Steinwart2012} for a thorough analysis), we have the eigenvalue decomposition
\[
	\scrT(f) = \sum_{\ell=1}^\infty \lambda_\ell \ipsub{v_\ell}{f}{L_2} v_\ell,
\]
where $\{ v_\ell \}_{\ell=1}^\infty$ is an orthonormal basis for $L_2(X, \mu)$,
and $\lambda_1 \geq \lambda_2 \geq \lambda_3 \geq \cdots$ are the eigenvalues of $\scrT$ arranged in decreasing order.
Furthermore, we have
\[
	k(x, y) = \sum_{\ell=1}^\infty \lambda_\ell v_\ell(x) v_\ell(y).
\]
We can handle the finite-dimensional case by setting $\lambda_\ell = 0$ for $\ell > d$, where $d = \dim(\scrH)$
(furthermore, the standard linear regression case can be recovered with $X = \R^d$ and $k(x, y) = \ip{x}{y}_{\ell_2}$).
Note that in order to interpolate an arbitrary set of samples, we need the dimension $d$ to be at least the number of samples $n$ (otherwise, the linear system is overdetermined).

We will also use the following well-known fact throughout our analysis:
for any $f, g \in L_2$, we have
\[
	\ipsub{f}{g}{L_2} = \ipH{\scrT^{1/2} f}{\scrT^{1/2} g}.
\]
Hence, $\scrT^{1/2}$ is an isometry from $L_2$ to $\scrH$.
Note that this implies that for every $f \in \scrH$,
\[
	\normH{f}^2 = \sum_{\ell=1}^\infty \frac{\ipsub{f}{v_\ell}{L_2}^2}{\lambda_\ell}.
\]
Intuitively, we expect that if $f$ has small/bounded $\scrH$-norm, most of its energy is captured by components corresponding to relatively large eigenvalues.
Therefore, it is feasible to recover an accurate (in $L_2$) estimate of $f$,
even though $f$ lies in an infinite-dimensional space.

We will assume $x_1, \dots, x_n \simiid \mu$, i.e., the training examples are drawn from the same measure as the test example $x \sim \mu$.
Since we are evaluating a regression task, we wish to bound the squared (excess) prediction loss $\E (\fhat(x) - f^*(x))^2 = \norm{\fhat - f^*}_{L_2}^2$.
We will provide non-asymptotic upper bounds on $\norm{\fhat - f^*}_{L_2}^2$ as a function of the number of training examples $n$.
We will also focus on understanding scenarios for which we obtain statistical consistency, i.e., $\norm{\fhat - f^*}_{L_2}^2 \to 0$ as $n \to \infty$.

\subsection{Main Results for Deterministic Sample Locations}
\label{sec:results_determ}
To state our main results, we introduce some additional notation.
Here and for the rest of this section, $p$ will be a fixed integer that we can tune in our analysis.
We divide the function space $L_2(X, \mu)$ into two parts:
$G = \spn\{v_1, \dots, v_p\}$ denotes the space spanned by the first $p$ eigenfunctions of $\scrT$, and $G^\perp$ denotes its orthogonal complement (in both $L_2$ and $\scrH$).
Accordingly, we split our sampling operator into two parts: $\scrA_G = \restr{\scrA}{G}$ and $\scrR = \restr{\scrA}{G^\perp}$.
Intuitively, if $p$ is chosen such that $\lambda_{p+1}, \lambda_{p+2}, \dots$ are relatively small,
we expect $G$ to contain most of the energy in any given function $f \in \scrH$.
A key fact is that the Gram matrix can be decomposed as $\scrA \scrA^* = \scrA_G \scrA_G^* + \scrR \scrR^*$.
The dimension $p$ is similar to (but more flexible than) the regularization-dependent effective dimension in \cite{Zhang2005,Hsu2014}.

Since $p = \dim(G)$ is finite,
we can recover a function in $G$ from a finite number of samples.
We state this quantitatively by analyzing the restricted sampling operator $\scrA_G$.
To state concentration results on $G$ in terms of the $L_2$ norm, we denote $\scrC = \scrA_G$,
and we let $\scrC^* = \scrT_G^{-1} \scrA_G^*$ be its adjoint with respect to the $L_2$ inner product.
Note that $\frac{1}{n} \E \scrA_G^* \scrA_G = \scrT_G$,
where $\scrT_G = \restr{\scrT}{G}$.
Therefore,
\[
	\frac{1}{n} \E \scrC^* \scrC = \scrI_G,
\]
where $\scrI_G$ is the identity operator on $G$.
Provided that $n \gg p$, we expect $\frac{1}{n} \scrC^* \scrC \approx \scrI_G$.
We will analyze how closely this holds later;
we first state \emph{deterministic} results that depend on the error in this approximation.

The second key approximation regards the ``remainder Gram matrix'' $\scrR \scrR^*$.
Previous interpolation literature has assumed that this matrix is approximately a multiple of the identity $I_n$ (or is in some sense ``well-conditioned'').
We will again analyze how accurately this holds later, but for now, we will state our main results assuming that $\alpha I_n + \scrR \scrR^*$ is upper and lower bounded by multiples of the identity.
There is no requirement that $\alpha \geq 0$;
in principle, our framework applies to negative regularization \citep{Tsigler2020}, but we do not explore this aspect in detail.

Finally, we will assume, for simplicity and brevity, that $f^* \in G$ exactly.
If this did not hold, there would be another term in the ``bias'' error bound whose size is directly proportional to the size of $\scrP_{G^\perp}(f^*)$, which in turn is negligible provided that $f^* \in \scrH$ (i.e., $f^*$ has bounded $\scrH$-norm).
Note that kernel methods run into fundamental approximation-theoretic limitations in the absence of a bounded-$\scrH$-norm assumption~\citep{Ghorbani2021,Donhauser2021,Belkin2018b}.

\begin{theorem}[Bias]
	\label{thm:reg_bias_deterministic}
	Suppose that
	\begin{enumerate}
	    \item $\alpha_L I_n \preceq \alpha I_n + \scrR \scrR^* \preceq \alpha_U I_n$ for some numbers $\alpha_U \geq \alpha_L > 0$, and
	    \item $\frac{\alpha_U - \alpha_L}{\alpha_U + \alpha_L} + \frac{2}{n} \norm*{ \scrC^*\scrC - n \scrI_G}_{L_2} \leq c$ for some $c < 1$.
	\end{enumerate}
	Let $\alphabr = \frac{2 \alpha_U \alpha_L}{\alpha_U + \alpha_L}$ be the harmonic mean of $\alpha_U$ and $\alpha_L$.
	Then, for any $f^* \in G$, we have
	\begin{align*}
		\norm*{ \fhat_0 - f^* }_{L_2}
		&\lesssim \min \braces*{\sqrt{\lambda_1},  \frac{1}{1-c}  \frac{\alphabr}{n \sqrt{\lambda_p}}, \frac{1}{1-c} \sqrt{\frac{\alphabr}{n} } } \\
		&\qquad\times \parens*{ 1 + \sqrt{\frac{n \lambda_{p+1}}{\alphabr}} } \normH{f^*}.
	\end{align*}
\end{theorem}

\begin{theorem}[Variance]
	\label{thm:reg_var_deterministic}
	Suppose the conditions of \Cref{thm:reg_bias_deterministic} hold,
	and let $\alphatl = \frac{\alpha_U + \alpha_L}{2}$.
	Furthermore, suppose the $\xi_i$'s are zero-mean and independent with variance bounded by $\sigma^2$.
    Then
	\begin{align}\label{eq:varianceerrorterms}
		\E_{\xi} \norm{\epsilon}_{L_2}^2
		\lesssim \sigma^2 \parens*{\frac{\alpha_U}{\alpha_L} + 1}^2 \parens*{ \frac{p}{n} + \frac{\tr_{L_2}(\scrR^* \scrR)}{\alphatl^2} }.
	\end{align}
\end{theorem}

\Cref{sec:proof_sketch} contains simplified proof sketches of \Cref{thm:reg_bias_deterministic,thm:reg_var_deterministic};
we provide complete proofs in \Cref{app:deterministic_proofs} in the supplementary material.
The reader should note that our proofs consist of relatively simple linear algebra.
Compare this, for example, to \cite{Bartlett2020} or \cite{Muthukumar2021},
where the analysis depends delicately on the independence (or, in the latter case, even Gaussianity) of the features via rather complicated matrix manipulations.

We could also obtain a high-probability (with respect to $\xi$) bound on the variance (if, e.g., the $\xi_i$'s are sub-Gaussian), but we omit this to preserve the clarity and simplicity of the result.
We outline how one could do this in \Cref{app:variance_proof}.


\subsection{Operator Concentration Results}
\label{sec:operator_conc}

We now state operator concentration results on three important quantities:
(a) the deviation of the residual Gram matrix $\scrR \scrR^*$ from a multiple of the identity, (b) the quantity $\tr_{L_2} (\scrR^* \scrR)$ which appears in the variance bound, and (c) the deviation of $\frac{1}{n} \scrC^* \scrC$ from $\scrI_G$.
All proofs are contained in \Cref{app:conc_proofs} in the supplementary material.
We begin with our most general results that apply under minimal assumptions.

\subsubsection{General Residual Concentration}\label{sec:generalresidual}

Let $k^R(x, y) = \sum_{\ell > p} \lambda_\ell v_\ell(x) v_\ell(y)$ be the
reproducing kernel restricted to $G^\perp$.
\begin{lemma}[Generic residual Gram matrix]
    \label{lem:R_conc_generic}
	\begin{align*}
		\E \norm{\scrR \scrR^* - (\tr \scrT_{G^\perp}) I_n}^2
		&\lesssim n^2 \tr(\scrT_{G^\perp}^2) \\
		&\qquad + \norm{k^R(\cdot, \cdot) - \tr \scrT_{G^\perp} }_\infty^2,
	\end{align*}
	where $\norm{k^R(\cdot, \cdot) - \tr \scrT_{G^\perp} }_\infty = \sup_x \{\abs{k^R(x, x) - \tr \scrT_{G^\perp} } \}$.
\end{lemma}
Note for this result to give $\alpha_L I_n \preceq \scrR \scrR^* \preceq \alpha_U I_n$
where $\alpha_U / \alpha_L$ is bounded,
we need $\tr \scrT_{G^\perp} \gtrsim n\sqrt{\tr(\scrT_{G^\perp}^2)}$.
Even when $\{\lambda_\ell\}_{\ell > p}$ are all equal (see \Cref{sec:bilevel}),
we need $\dim{\scrH} = d \gtrsim n^2$. While this may seem restrictive, it is not possible to do better without additional assumptions on the features. In \Cref{sec:FourierExample}, we show that in the case of Fourier features, $\lambda_{\text{max}}(\scrR\scrR^*)/\lambda_{\text{min}}(\scrR\scrR^*) \gtrsim \tfrac{n^4}{\tau^2d^2}$ with probability at least $1-e^{-\tau}$, and thus $d \gtrsim n^2$ is necessary.
This can be significantly relaxed when the features are independent, as shown in \Cref{sec:ind_feat}.

To bound the variance, we will use the following expectation throughout the rest of this paper:
\begin{lemma}[Generic trace bound on $\scrR^* \scrR$]
	\label{lem:R_trace_exp}
	\[
		\E \tr_{L_2}(\scrR^* \scrR) = n \tr(\scrT_{G^\perp}^2) = n \sum_{\ell > p} \lambda_\ell^2.
	\]
\end{lemma}
Note that \Cref{lem:R_trace_exp}, \Cref{thm:reg_var_deterministic},
and the approximate identity $\scrR \scrR^* \approx (\tr \scrT_{G^\perp}) I_n$ combine to bound the variance error as $\norm{\epsilon}_{L_2}^2 \lesssim \frac{p}{n} + n \parens*{\sum_{\ell > p} \lambda_\ell^2} / \parens*{\sum_{\ell > p} \lambda_\ell}^2$.
This is identical to the bound provided in \cite{Bartlett2020}.


\subsubsection{Bounded Orthonormal System}
Our results show that harmless interpolation can occur in much more general settings than independent and/or sub-Gaussian features.
An important class of features that are not independent or sub-Gaussian is a \emph{bounded orthonormal system} (BOS).

On the subspace $G$ defined before, the basis $v_1, \dots, v_p$ is a BOS if it is an orthonormal basis in $L_2$ (as we have already assumed) and, further, we have
\[
    \sum_{\ell = 1}^p v_\ell^2(x) \leq C p
\]
$\mu$-almost surely in $x$ for some constant $C \geq 1$.
Equivalently, for all $f \in G$, $\norm{f}_\infty^2 \leq Cp \norm{f}_{L_2}^2$.

This assumption is satisfied by many popular choices of features including sinusoids (see \Cref{sec:numerical}), Chebyshev polynomials, and the standard Euclidean basis on $\R^d$.
One can also often show that kernel eigenfunctions satisfy this property, such as when the data lie on a low-dimensional manifold \citep{McRae2020}.

It is easy to derive concentration inequalities for bounded orthonormal systems via matrix/operator concentrations results for sums of bounded independent random matrices (e.g., \cite{Tropp2015}---see our supplementary material for details).
A bound that is useful for our purposes is the following:
\begin{lemma}[BOS sampling operator on $G$]
	\label{lem:C_conc_BOS}
	If $G$ is spanned by a bounded orthonormal system with constant $C$,
	then, for $t > 0$, with probability at least $1 - e^{-t}$,
	\[
	\frac{1}{n} \norm{ \scrC^* \scrC - n \scrI_G}_{L_2} \lesssim \sqrt{\frac{C p (t + \log p)}{n}} + \frac{Cp (t + \log p)}{n}.
	\]
\end{lemma}
Thus if $n \gtrsim Cp \log p$, we can have, say, $\frac{1}{n} \norm{ \scrC^* \scrC - n \scrI_G}_{L_2} \leq 1/4$ (or any other small constant) with high probability.

In general, the $C p \log p$ sample complexity is optimal under the BOS assumption.
As a simple example, consider the following basis $\{v_1, \dots, v_p\}$ on $\R^p$ (written as functions on $\{1, \dots, p\}$):
for uniquely determined constants $c_1$ and $c_2$,
set the measure to be $\mu(\{j\}) = \frac{1}{Cp}$ for $j < p$ and $\mu(\{p\}) = c_1$,
and set $v_\ell = \sqrt{C p} \delta_{\ell}$ for $\ell < p$ and $v_p = c_2 \delta_p$.
One can easily verify by a coupon collector argument that we need $O(Cp \log p)$ samples from $\mu$ merely to sample every coordinate at least once.



\subsubsection{Independent Features}\label{sec:ind_feat}
To compare to prior work, we
list independent-feature concentration results that can be plugged into our \Cref{thm:reg_bias_deterministic}.
Suppose that for $x \sim \mu$,
the features $\{v_\ell(x)\}$ are independent random variables.
The key benefit this gives us is that we can now write the residual Gram matrix $\scrR \scrR^*$ as a sum of \emph{independent} random rank-1 matrices.
To see this, define the vectors
\[
	w_\ell = (v_\ell(x_1), v_\ell(x_2), \dots, v_\ell(x_n)) \in \R^n.
\]
We have already been assuming that the entries of each $w_\ell$ are independent (since they only depend on the independent variables $x_i$),
but an independent features assumption implies that the entire set of random vectors $\{w_\ell\}_{\ell \geq 1}$ is independent.
We can then write
\[
	\scrR \scrR^* = \sum_{\ell > p} \lambda_\ell w_\ell \otimes w_\ell.
\]
We state a formal result for sub-Gaussian independent features.
We expect similar results hold for much weaker tail conditions.
\begin{lemma}[Independent features residual Gram matrix]
	\label{lem:R_conc_ind}
	Suppose the features $\{v_\ell(x)\}_{\ell \geq 1}$ are zero-mean, independent, and sub-Gaussian.
	Then, for $t > 0$, with probability at least $1 - e^{-t}$,
	\[
		\norm{\scrR \scrR^* - (\tr \scrT_{G^\perp}) I_n } \lesssim \sqrt{(n + t) \tr(\scrT_{G^\perp}^2)} + (n + t) \lambda_{p+1}.
	\]
\end{lemma}
The zero-mean assumption is for simplicity and can easily be relaxed at the cost of a more complicated theorem statement.
Note that this is stronger than \Cref{lem:R_conc_generic} in two ways:
first, the bound holds with exponentially high probability as opposed to being merely in expectation.
Second, we have effectively replaced the $n^2$ in \Cref{lem:R_conc_generic} by $n$,
greatly reducing the amount of overparametrization we need.

Note for this result to give $\alpha_L I_n \preceq \scrR \scrR^* \preceq \alpha_U I_n$
where $\alpha_U / \alpha_L$ is bounded,
we need
\[
	n \lesssim \frac{\tr \scrT_{G^\perp}}{\lambda_{p+1}} = \frac{1}{\lambda_{p+1}} \sum_{\ell > p} \lambda_\ell
\]
(this also gives us $\sqrt{n \tr(\scrT_{G^\perp}^2)} \lesssim \tr \scrT_{G^\perp} $ by Cauchy-Schwartz).
This is identical to the requirement that $r_{k^*}(\Sigma) \geq b n$ in \cite{Bartlett2020}.

We can also obtain slightly improved results (vs.\ the BOS assumption) for concentration of $\scrC^* \scrC$:
\begin{lemma}[Sampling operator on $G$ under independent features]
    With probability at least $1 - e^{-t}$,
    \[\norm*{\frac{1}{n} \scrC^* \scrC - \scrI_G}_{L_2}
    \lesssim \sqrt{\frac{p+t}{n}} + \frac{p+t}{n}.\]
\end{lemma}
Thus we only require $n \gtrsim p$ to obtain the required concentration.
For a proof, see, for example, \cite[Section 4.6]{Vershynin2018}.


\subsection{Informal Proof Sketch (Deterministic)}
\label{sec:proof_sketch}
Here we outline the basic proof structure of \Cref{thm:reg_bias_deterministic,thm:reg_var_deterministic}.
We will perform the analysis as though $\alpha I_n + \scrR \scrR^* = \alphabr I_n$ and $\scrC^* \scrC = n \scrI_G$ (equivalently, $\scrA_G^* \scrA = n \scrT_G$),
and we will write ``$\approx$'' where we make these substitutions.
The main additional steps we need are to quantify the error due to these approximations.

\subsubsection{Bias Term (from Signal)}
Note that since $f^* \in G$, we have
\begin{align*}
	\fhat_0
	&= \scrA^* (\alpha I_n + \scrA \scrA^*)^{-1} \scrA_G f^* \\
	&\approx \scrA^* (\alphabr I_n + \scrA_G \scrA_G^*)^{-1} \scrA_G f^* \\
	&= \scrA^* \scrA_G (\alphabr \scrI_G + \scrA_G^* \scrA_G)^{-1} f^* \\
	&\approx \begin{bmatrix*} \scrA_G^* \\ \scrR^* \end{bmatrix*} \scrA_G (\alphabr \scrI_G + n \scrT_G)^{-1} f^*.
\end{align*}
From this we obtain
\begin{align*}
	\scrP_G(\fhat_0)
	&= \scrA_G^* \scrA_G (\alphabr \scrI_G + n \scrT_G)^{-1} f^* \\
	&\approx n \scrT_G (\alphabr \scrI_G + n \scrT_G)^{-1} f^* \\
	&= \scrSbr f^*,
\end{align*}
where $\scrSbr \coloneqq n \scrT_G (\alphabr \scrI_G + n \scrT_G)^{-1}$ is the idealized ``survival'' operator, representing the extent to which the original signal $f^*$ is preserved.
We then have $f^* - \scrP_G(\fhat_0) \approx \scrBbr f^*$,
where $\scrBbr \coloneqq \scrI_G - \scrSbr = \alphabr (\alphabr \scrI_G + n \scrT_G)^{-1}$ is the idealized ``bias'' operator.
One can 
verify that
\[
	\norm{\scrBbr}_{\scrH \to L_2} \lesssim \min\braces*{\sqrt{\lambda_1}, \frac{\alphabr}{n \sqrt{\lambda_p}}, \sqrt{\frac{\alphabr}{n}} }.
\]
This bounds $\norm{\scrP_G(\fhat_0) - f^*}_{L_2}$ for \Cref{thm:reg_bias_deterministic};
the formal theorem has an extra factor of $1/(1-c)$ which comes from the approximation errors (recall that $c$ is determined by how accurate our idealizing approximation are---see the statement of \Cref{thm:reg_bias_deterministic} for the precise definition).

Next, note that $\scrP_{G^\perp}(\fhat_0) \approx \scrR^* \scrC \frac{\scrBbr}{\alphabr} f^*$, where we have substituted $\scrC$ for $\scrA_G$.
Because
\[
	\normsub{\scrR^*}{\ell_2 \to L_2} \leq \normsub{\scrI_{G^\perp}}{\scrH \to L_2} \normsub{\scrR^*}{\ell_2 \to \scrH}
	\lesssim \sqrt{\lambda_{p+1} \alphabr},
\]
we have
\begin{align*}
    \normsub*{\scrR^* \scrC \frac{\scrBbr}{\alphabr}}{\scrH \to L_2}
    &\leq \normsub{\scrR^*}{\ell_2 \to L_2} \normsub{\scrC}{L_2 \to \ell_2} \frac{\normsub{\scrBbr}{\scrH \to L_2}}{\alphabr} \\
    &\lesssim \sqrt{\frac{n \lambda_{p+1}}{\alphabr}} \normsub{\scrBbr}{\scrH \to L_2},
\end{align*}
which allows us to bound $\normsub{\scrP_{G^\perp}(\fhat_0)}{L_2}$.

\subsubsection{Variance Term (from Noise)}
Making similar approximations as above (with $\alphatl$ instead of $\alphabr$ -- the distinction comes from the approximation arguments we use in the formal proof), we have
\begin{align*}
	\epsilon
	&= \scrA^* ( \alpha I_n + \scrA \scrA*)^{-1} \xi \\
	&\approx \begin{bmatrix*} \scrA_G^* \\ \scrR^* \end{bmatrix*} ( \alphatl I_n + \scrA_G \scrA_G^*)^{-1} \xi \\
	&= \begin{bmatrix*}
		(\alphatl \scrI_G + \scrA_G^* \scrA_G)^{-1} \scrA_G^* \xi \\
		\scrR^* ( \alphatl I_n + \scrA_G \scrA_G^*)^{-1} \xi
	\end{bmatrix*} \\
	&\approx \begin{bmatrix*}
		(\alphatl \scrT_G^{-1} + n \scrI_G)^{-1} \scrC^* \xi \\
		\scrR^* ( \alphatl I_n + \scrA_G \scrA_G^*)^{-1} \xi
	\end{bmatrix*}.
\end{align*}
Therefore,
\begin{align*}
	\E_\xi \norm{\epsilon}_{L_2}^2
	&\approx \sigma^2 \norm*{ \begin{bmatrix*}
			(\alphatl \scrT_G^{-1} + n \scrI_G)^{-1} \scrC^* \\
			\scrR^* ( \alphatl I_n + \scrA_G \scrA_G^*)^{-1}
	\end{bmatrix*} }_{HS, \ell_2 \to L_2}^2 \\
	&\lesssim \sigma^2 \parens*{ \frac{1}{n^2} \tr_{L_2} (\scrC^* \scrC) + \frac{1}{\alphatl^2} \tr_{L_2}(\scrR^* \scrR) } \\
	&\approx \sigma^2 \parens*{ \frac{1}{n^2} \tr_{L_2} (n \scrI_G) + \frac{1}{\alphatl^2} \tr_{L_2}(\scrR^* \scrR) } \\
	&= \sigma^2 \parens*{ \frac{p}{n} + \frac{1}{\alphatl^2} \tr_{L_2}(\scrR^* \scrR) }.
\end{align*}
The factor of $\alpha_U/\alpha_L$ comes from the approximation arguments.

\section{KERNEL CLASSIFICATION}
\label{sec:class}
We now consider the case of classification, in which the observation $y$ is a (noisy) binary label in $\{-1, 1\}$ with a distribution depending on $x$.
Our approach is to perform ordinary linear/kernel regression on the binary labels $y_i$ with the squared loss function.
Although this seems counter-intuitive,
recent results (e.g., \cite{Hui2021}) have shown that training with the squared-loss is highly competitive with the more common cross-entropy loss function in practice.
Separately, recent results have also shown that regression on binary labels is, in some interesting overparametrized cases,
equivalent to the maximum-margin SVM (e.g., \cite{Muthukumar2021,Hsu2021}). 
Inspired by these findings, we study the minimum-$\ell_2$-norm interpolator of the binary labels $\{y_i\}_{i=1}^n$ and its ensuing \emph{classification} error.

Through the lens of regression,
our target function $f^*$ is now replaced by
\[
\eta^*(x) \coloneqq \E (y \given x) = 2 \P(y = 1 \given x) - 1.
\]
The label noise is $\xi = y - \eta^*(x)$. Note that $\E[\xi|x] = 0$ by definition, and $\var(\xi \given x) = 1 - (\eta^*)^2(x)$.
Our assumption on the label noise model is that $\eta^* \in G$.

The regression procedure yields an estimator $\etahat$ of $\eta^*$.
Then, our classification rule is given by $\yhat = \sign(\etahat)$.
Given a probability distribution $\mu$ over $x$,
the \emph{excess risk} of the classification rule with respect to the Bayes-optimal classifier is given by
\[
\scrE \coloneqq \P(\yhat \neq y) - \P(y \neq \sign (\eta^*)).
\]
Standard calculations
 (see \cite{Friedman1997}) give
\[
\scrE = \int \abs{\eta^*(x)} \indicator{ \sign(\etahat(x)) \neq \sign(\eta^*(x))} d\mu(x).
\]
Thus the excess risk is the average of the sign error of $\etahat$ versus $\eta^*$ modulated by how distinguishable the two classes are (which is represented by $\abs{\eta^*}$).

To bound $\scrE$,
we decompose our estimate $\etahat$ as
\begin{equation}
	\label{eq:est_decomp}
	\etahat = s \eta^* + \etahat_r,
\end{equation}
where $s$ is a parameter that we can tune in our analysis, and $\etahat_r$ is the residual.
If $s > 0$, we have
\[
\braces{ \sign(\etahat) \neq \sign(\eta^*) } \subseteq \braces{ \abs{\etahat_r} \geq s \abs{\eta^*} },
\] 
so
\begin{equation}
	\label{eq:riskboundratio}
	\scrE \leq \frac{1}{s} \int \abs{\etahat_\perp(x)} \ d\mu(x) = \frac{\norm{\etahat_r}_{L_1}}{s} \leq \frac{\norm{\etahat_r}_{L_2}}{s},
\end{equation}
where the norm inequality is due to the fact that $\mu$ is a probability measure.
For reasons that will shortly become clear, 
we will call $s$ the \emph{survival factor} and $\etahat_r$ the \emph{residual}.

A first possible choice for the quantities in \eqref{eq:est_decomp} is $s = 1$ and $\etahat_r = \etahat - \eta^*$.
This choice yields $\scrE \leq \norm{\etahat - \eta}_{L_1}$; therefore, small regression error implies small excess classification risk.
However, we are interested in cases in which the regression error is not small but the classification error is.
To use the bound \eqref{eq:riskboundratio}, we would need to show that we can have the ratio $\norm{\etahat_r}_{L_2}/s$ be very small with a different choice of $s \ll 1$.

We now show how this can work.
We recall the idealized ``survival" and ``bias" operators $\scrSbr$ and $\scrBbr$ from \Cref{sec:proof_sketch}.
Note that to bound the regression error we show that $\etahat \approx \scrSbr(\eta^*)$ and that $\norm{\eta^* - \scrSbr\eta^*}_{L_2} = \norm{\scrBbr \eta^*}_{L_2}$ is small.
For the classification problem, an interesting new possibility arises.
As a simple example, suppose all the first $p$ eigenvalues $\lambda_1, \dots, \lambda_p$ are identically $1$.
Then $\scrSbr = \frac{n}{\alphabr + n} \scrI_G$.
If $\alphabr \ll n$,  then the ideal bias $\scrBbr = \frac{\alphabr}{\alphabr + n} \scrI_G$ will be small.
However, what if $\alpha \gtrsim n$, in which case the bias is not small?
We cannot get small regression error, but for classification,
we can apply~\eqref{eq:riskboundratio} while choosing $s = \frac{n}{\alphabr + n}$.
Then, as long as
\[
    \norm{\etahat - \scrSbr \eta^*}_{L_2} \ll \frac{n}{\alphabr + n},
\]
we will have small excess classification risk.
In~\Cref{sec:bilevel}, we use this observation to provide sufficient conditions for classification consistency, and demonstrate that these conditions are significantly weaker than the ones needed to be regression-consistent.
This apporach is qualitatively similar to that of \cite{Muthukumar2021},
which provides a slightly sharper bound but relies on a special form of $\eta^*$ and Gaussianity of the features.
Their techniques do not easily extend to a more general setting.

To combine this framework with our previous interpolation results,
note that, under our new notation, $\etahat = \etahat_0 + \epsilon$,
where $\etahat_0 = \scrA^* (\scrA \scrA^*)^{-1} \scrA \eta^*$ and $\epsilon = \scrA^* (\scrA \scrA^*)^{-1} \xi$.
We will show that $\etahat_0 \approx \scrSbr \eta^*$ and $\epsilon$ is small.
For the first objective,
we present here a more refined version of \Cref{thm:reg_bias_deterministic} that bounds the error to $\scrSbr \eta^*$ rather than $\eta^*$ itself.
This will be used to characterize the classification error in~\Cref{sec:bilevel}.
\begin{lemma}[More refined bias estimate]
    \label{lem:approx_error}
    Under the conditions of \Cref{thm:reg_bias_deterministic}
    (assuming $c$ is bounded away from 1 so that $(1 - c)^{-1}$ is subsumed into the constants),
    \begin{align*}
        \norm{\etahat_0 - \scrSbr \eta^*}_{L_2}
        &\lesssim \parens*{ c + \sqrt{\frac{n \lambda_{p+1}}{\alphabr}} } \\
        &\qquad \times \min\braces*{\lambda_1, \frac{\alphabr}{n \sqrt{\lambda_p}}, \sqrt{\frac{\alphabr}{n}} } \normH{\eta^*}.
    \end{align*}
\end{lemma}
The proof of Lemma~\ref{lem:approx_error} is an easy modification of the proof of \Cref{thm:reg_bias_deterministic} (see \Cref{app:bias_proof}).

\subsection{Bi-level Ensemble Asymptotic Analysis}
\label{sec:bilevel}
We now examine the implications of this refined classification analysis in a bounded orthonormal system (BOS).
In particular, suppose that the eigenfunctions are all bounded (e.g., a Fourier series for periodic functions on an interval).
For the eigenvalues, we consider the bi-level ensemble as defined in \cite{Muthukumar2021} with non-negative parameters $n, \beta, q, r$ (where $\beta > 1$ and $r < 1$).
This ensemble contains $d = n^\beta$ features, of which $p = n^r$ have ``large'' eigenvalues, and the remaining $d - p$ eigenvalues are small and their relative magnitude depends on the parameter $q$.
Specifically, we set
\begin{align}\label{eq:bilevel}
\lambda_\ell = \begin{cases}
	1, & 1 \leq \ell \leq p \\
	n^{-(\beta - r - q)}, & p < \ell \leq d.
\end{cases}
\end{align}
We require $q < \beta - r$ to ensure that the ``small'' eigenvalues are actually smaller than 1.

\begin{corollary}
    \label{cor:bilevel_asym}
    Consider the bi-level ensemble with parameters $n, \beta, q, r$, and suppose that the eigenfunctions are all bounded by an absolute constant.
    Further, suppose that $\beta > 2$ and $r < 1$,
    and $\eta^* \in G$.
    Then we obtain the following asymptotic results as $n \to \infty$:
    \begin{enumerate}
        \item If $q < 1 - r$, as $n \to \infty$, $\norm{\etahat - \eta^*}_{L_2} \to 0$ in probability, and therefore both regression and classification are consistent.
        \item If $q > 1 - r$, $\norm{\etahat}_{L_2} \to 0$ in probability, and therefore regression is inconsistent for nonzero $\eta^*$.
        \item If $q < \frac{3}{2}(1 - r)$ and $\beta > 2(r + q)$,
        excess classification risk $\scrE \to 0$ in probability, that is, classification is consistent.
    \end{enumerate}
\end{corollary}
\Cref{cor:bilevel_asym} is proved in \Cref{app:bilevel_proof}.
Note that we have an asymptotic separation between classification and regression when $1 - r < q < \frac{3}{2}(1 - r)$.
This is comparable to the results of \cite{Muthukumar2021}, which allow slightly larger $q$ and smaller $\beta$ but require much stronger feature assumptions.

We use the bi-level ensemble model in~\eqref{eq:bilevel} for simplicity; however, we can obtain non-asymptotic bounds on the classification risk under more general assumptions.
For a fixed value of $p := n^r$, our analysis allows the tail eigenvalues corresponding to indices $p < \ell \leq d$ to be non-uniform.
The requirement that the top $p$ eigenvalues are the same is somewhat more stringent;
in general, when the eigenvalues are different, the survival operator $\scrSbr$ is not a multiple of the identity.
This could lead to qualitatively different behavior, as now $\etahat$ may be distorted from $\eta^*$ due to differences in the eigenvalues of $\scrT_G$.
This problem disappears in the case that either (a) $\eta^*$ is proportional to a single eigenfunction or (b) the first $p$ eigenvalues of $\scrT$ are identical (both of which hold in \cite{Muthukumar2021}).
We analyze further the extent to which we can bound the classification risk when \emph{neither} of these assumptions holds in \Cref{sec:distortion}.
Our analysis method requires $\lambda_p$ to be close to $\lambda_1$ to obtain significant gains for classification over regression.

\section{NUMERICAL EXPERIMENTS}
\label{sec:numerical}
We now perform numerical experiments to demonstrate how the parameters $\beta, r,$ and $q$ of the bi-level ensemble model affect regression and classification performance. We consider the case of Fourier features $v_{\ell}(x) = e^{\mathbf{j} 2\pi \ell x}$ for $\ell = -d,\ldots,d$ over $x \in [0,1]$ with the uniform sampling measure, and the bi-level ensemble as defined in~\eqref{eq:bilevel}.
The corresponding kernel function is
\begin{align*}
    k(x,y) &= \sum_{\ell = -d}^{d}\lambda_{\ell}v_{\ell}(x)\overline{v_{\ell}(y)}
    \\
    &= (1-n^{-(\beta-r-q)})D_{p}(x-y) \\
    &\qquad+ n^{-(\beta-r-q)}D_d(x-y), 
\end{align*}
where $D_m(t) = \tfrac{\sin[(2m+1)\pi t]}{\sin(\pi t)}$ is the Dirichlet sinc function.
We consider three cases for the bi-level ensemble parameters: $(\beta,r,q) = (2.6,0.3,0.3)$, $(2.6,1/3,5/6)$, and $(2.6,0.8,0.45)$.
We sweep over several values of $n$ between $10$ and $3162$. For each $n$, we generate an $\eta^* \in \text{span}\{v_{\ell}\}_{\ell = -p}^{p}$, scaled such that $\displaystyle\max_{x \in [0,1]}|\eta^*(x)| = 1$. 

We first attempt to reconstruct $\eta^*(x)$ from noisy samples $y^{\text{reg}}_i = \eta^*(x_i)+\xi_i$ for $i = 1,\ldots,n$ where $\xi_i$ are i.i.d.\ $\mathcal{N}(0,1)$. We use the kernel ridge regression estimator $\widehat{\eta^*}^{\text{reg}} = \scrA^*(\alpha I_n + \scrA\scrA^*)^{-1}y^{\text{reg}}$ with a regularization parameter of $\alpha = 10^{-3}$. We then measure the relative $L_2$ error of the estimate, i.e., $\scrE^{\text{reg}} = \|\eta^*-\widehat{\eta}^{\text{reg}}\|_{L_2}^2/\|\eta^*\|_{L_2}^2$.

We also attempt to reconstruct $\eta^*(x)$ from binary observations $y^{\text{class}}_i = +1$ with probability $(1+\eta^*(x_i))/2$ and $-1$ with probability $(1-\eta^*(x_i))/2$ for $i = 1,\ldots,n$. We use the estimator $\widehat{\eta}^{\text{class}} = \scrA^*(\alpha I_n + \scrA\scrA^*)^{-1}y^{\text{class}}$ with a regularization parameter of $\alpha = 10^{-3}$. We then measure the relative excess risk, i.e., \[
\scrE^{\text{class}} = \dfrac{\int \abs{\eta^*(x)} \indicator{ \sign(\etahat^{\text{class}}(x)) \neq \sign(\eta^*(x))} \,dx}{\int \abs{\eta^*(x)}\,dx}.
\] The above procedure is repeated over $100$ trials. In Figure~\ref{fig:RegVsClass}, we plot the relative $L_2$ error (averaged over $100$ trials) versus $n$ and the relative excess risk (averaged over $100$ trials) versus $n$ for each of the three sets of values for $\beta, r, q$. In the first case where $\beta = 2.6$, $r = 0.3$, and $q = 0.3$, we have $r+q < 1$ and both $\scrE^{\text{reg}}$ and $\scrE^{\text{class}}$ decrease as $n$ increases. In the second case where $\beta = 2.6$, $r = 1/3$, and $q = 5/6$, we have $1-r < q < \tfrac{3}{2}(1-r)$ and $\beta > 2r+2q$, and $\scrE^{\text{class}}$ decreases as $n$ increases, but $\scrE^{\text{reg}}$ does not decrease as $n$ increases. In the third case where $\beta = 2.6$, $r = 0.8$, and $q = 0.45$, we have that $r+q > 1$, and $q > \tfrac{3}{2}(1-r)$, and both $\scrE^{\text{reg}}$ and $\scrE^{\text{class}}$ do not decrease as $n$ increases.

\begin{figure}
    \centering
    \includegraphics[width = 0.9\columnwidth]{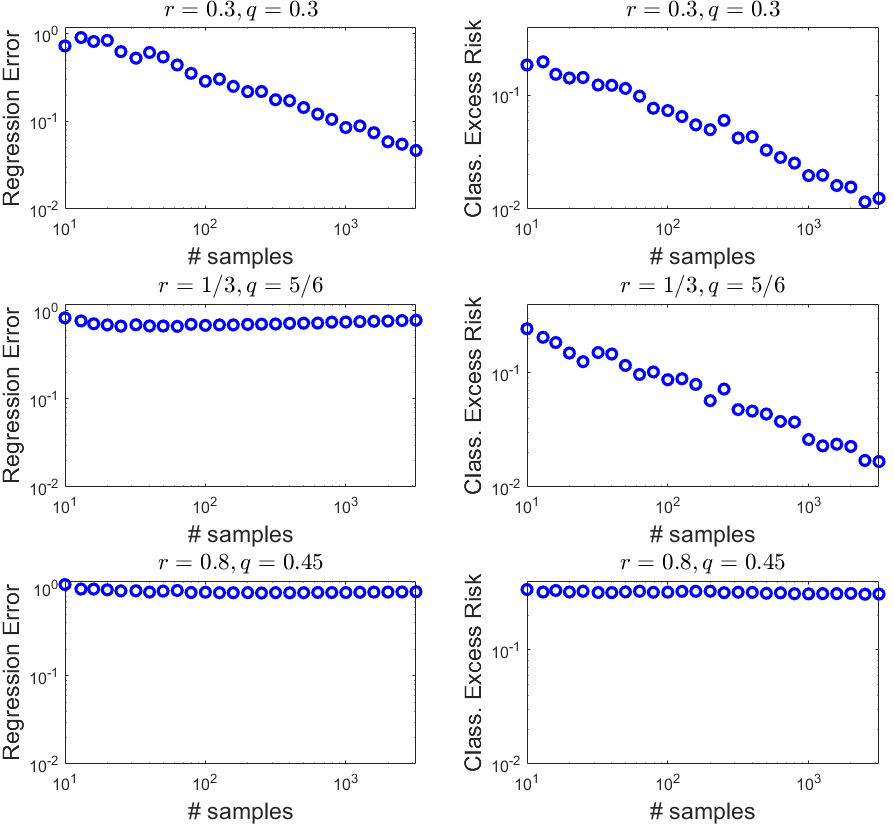} 
    \vspace{-2mm}
    \caption{\small Relative $L_2$ errors versus $n$ and the relative classification excess risks versus $n$ for each of the three sets of bi-level ensemble parameters (averaged over $100$ trials).} \vspace{-5mm}
    \label{fig:RegVsClass}
\end{figure}

\section{DISCUSSION}

In this paper we showed that under minimal assumptions on the data and feature map (a) harmless interpolation of noise in data is possible, and (b) we can be classification-consistent in high-dimensional regimes where we are not regression-consistent.
Important future directions include considering more general function models (e.g., any $f^* \in \scrH$ or even $f^* \not\in \scrH$),
better understanding the implications of distortion among the top eigenfunctions in classification error, and improving the non-asymptotic rates for classification risk from~\Cref{sec:bilevel}.
Another intriguing question is whether there is an equivalence between interpolating binary labels and the max-margin SVM (as shown in~\cite{Muthukumar2021,Hsu2021}) in the more general settings considered in this paper.
Finally, it would be very interesting to study whether our upper bounds (particularly for classification) can be matched by non-asymptotic lower bounds.

\subsubsection*{Acknowledgments}
We thank Fanny Yang for useful discussion.
This work was supported, in part, by National Science Foundation grant CCF-2107455.

\printbibliography


\clearpage
\appendix

\thispagestyle{empty}

\onecolumn \makesupplementtitle

\section{NOTATION}
\begin{table}[b!]
	\centering
	\caption{Notation}
	\label{tab:notation}
	\begin{tabular}{c l m{15em}}
		\toprule
		Symbol(s) & Definition(s) & Description \\ \midrule
		$k_x$ & $k_x = k(\cdot, x)$ & Kernel function centered at $x$ \\
		$\scrT$ & $\scrT(f) = \int f(x) k_x \ d\mu(x)$ & Integral operator of kernel $k$ \\
		$\{(\lambda_\ell, v_\ell)\}_{\ell=1}^\infty$ & $\scrT(f) = \sum_{\ell=1}^\infty \lambda_\ell \ipsub{f}{v_\ell}{L_2} v_\ell,~\lambda_1 \geq \lambda_2 \geq \cdots$ & Eigenvalue decomposition of $\scrT$ \\
		$\scrA$ & $\scrA(f) = \begin{bmatrix*} f(x_1) \\ \vdots \\ f(x_n) \end{bmatrix*}$ & Sampling operator from $\scrH$ to $\R^n$ \\
		$\scrA^*$ & $\scrA^*(z) = \sum_{i=1}^n z_i k_{x_i}$ & Adjoint of $\scrA$ w.r.t\ $\scrH$ and $\ell_2$ inner products \\
		$G$, $G^\perp$ & $G = \spn\{v_1, \dots, v_p\}$ & Span of first $p$ eigenfunctions of $\scrT$ (and its complement) \\
		$\scrI$ ($\scrI_G$) & & Identity operator (restricted to $G$) \\
		$\scrT_G, \scrT_{G^\perp}$ & $\scrT_G = \scrT \scrP_G$, $\scrT_{G^\perp} = \scrT \scrP_{G^\perp}$ & $\scrT$ restricted to $G$ and $G^\perp$ \\
		$\scrA_G$, $\scrR$ & $\scrA_G = \scrA \scrP_G, \scrR = \scrA \scrP_{G^\perp}$ & Restrictions of sampling operator to $G$, $G^\perp$ \\
		$\scrC, \scrC^*$ & $\scrC = \scrA_G$, $\scrC^* = \scrT_G^{-1} \scrA_G^*$ & Sampling operator and its adjoint on $G$ w.r.t.\ $L_2$ inner product on $G$ \\
		$\alpha$ & & Explicit regularization parameter \\
		$\alpha_L, \alpha_U$ & $\alpha_L I_n \preceq \alpha I_n + \scrR \scrR^* \preceq \alpha_U I_n$ & Lower and upper bounds on explicit+implicit regularization \\
		$\alphabr,\alphatl$ & $\alphabr = \frac{2 \alpha_U \alpha_L}{\alpha_U + \alpha_L}, \alphatl = \frac{\alpha_U + \alpha_L}{2}$ & Harmonic and arithmetic means of $\alpha_U, \alpha_L$ \\
		$\scrB$ & $\scrB = (\scrI_G + \scrA_G^* (\alpha + \scrR \scrR^*)^{-1} \scrA_G)^{-1}$ & Bias operator on $G$ \\
		$\scrS$ & $\scrS = \scrI_G - \scrB$ & Kernel regression operator (``survival'') on $G$ \\
		$\scrBbr$ & $\scrBbr = \parens*{\scrI_G + \frac{n}{\alphabr} \scrT_G}^{-1}$ & Idealized approximation to bias $\scrB$ \\
		$\scrSbr$ & $\scrSbr = \scrI_G - \scrBbr = \frac{n}{\alphabr} \scrT_G \parens*{\scrI_G + \frac{n}{\alphabr} \scrT_G}^{-1}$ & Idealized approximation to survival $\scrS$ \\
		\bottomrule
	\end{tabular}
\end{table}
For convenience in reading, we collect all notation that is used for the proofs in Table~\ref{tab:notation}.

In addition, we will use many different norms.
For a function $f \colon X \to \R$, $\norm{f}_{L_p} \coloneqq \parens*{ \parens*{ \E_{x \sim \mu} \abs{f(x)}^p }^{1/p} }$.
For $f \in \scrH$, $\normH{f} = \norm{\scrT^{-1/2} f}_{L_2}$ is the RKHS norm.
For $u \in \R^n$, $\norm{u}_{\ell_2}$ is the standard Euclidean norm.
We denote the $L_2$, $\scrH$, and $\ell_2$ inner products by $\ip{\cdot}{\cdot}_{L_2}$, $\ipH{\cdot}{\cdot}$, and $\ip{\cdot}{\cdot}_{\ell_2}$, respectively.

$\norm{\cdot}_{L_2}$, $\normH{\cdot}$, and $\norm{\cdot}_{\ell_2}$ also denote operator norms when applied to operators from the corresponding Hilbert space to itself.
We will write the operator norm of an operator $T\colon H_1 \to H_2$ (for any Hilbert spaces $H_1$ and $H_2$) with respect to the $H_1$ and $H_2$ norms as $\norm{T}_{H_1 \to H_2}$.
Similarly, $\norm{T}_{HS, H_1 \to H_2}$ refers to the Hilbert-Schmidt norm of $T$ with respect to the $H_1$ and $H_2$ inner products.

\section{DETAILED INTERPOLATION LITERATURE SURVEY}
\label{app:litreview}
Recent work has shown that the ``harmless interpolation" phenomenon becomes more pronounced with increased (effective) overparameterization when the minimum-Hilbert-norm interpolator is used in kernel regression~\citep{Liang2020,Liang2020a} or the minimum-norm interpolator is used in linear regression~\citep{Bartlett2020,Belkin2019a,Hastie2019,Muthukumar2020a,Tsigler2020,Mei2021,Mei2021a,Adlam2020,Ba2020,Dhifallah2020,DAscoli2020,Gerace2020,Hu2020,Li2021,Liao2020,Lin2021} in a variety of models.
See \cite{Belkin2021,Bartlett2021,Dar2021} for recent surveys of this line of work.

All of these models make at least one of the following assumptions: (a) independence of features~\citep{Bartlett2020,Hastie2019,Muthukumar2020a,Chinot2020}, (b) sub-Gaussianity in the feature vector~\citep{Bartlett2020,Tsigler2020}, (c) high data dimension~\citep{Hastie2019,Mei2021,Mei2021a,Liang2020,Liang2020a,Adlam2020,Ba2020,Dhifallah2020,DAscoli2020,Gerace2020,Hu2020,Li2021,Liao2020}, or (d) explicit structure in the kernel operator/feature map~\citep{Belkin2019a,Liang2020,Liang2020a,Muthukumar2020a,Mei2021,Mei2021a,Adlam2020,Ba2020,Dhifallah2020,DAscoli2020,Gerace2020,Hu2020,Li2021,Liao2020,Lin2021}. 
For specific kernels like the Laplace kernel, statistically consistent interpolation may actually \emph{require} growing data dimension with the number of training examples~\citep{Rakhlin2019}, as the data dimension fundamentally alters the eigenvalues of the Laplace kernel integral operator.
In contrast, our results do not explicitly posit any of these assumptions.
Our sufficient conditions for harmless interpolation are expressed purely in terms of the eigenvalues of the kernel integral operator and do not require special structure either on the eigenfunctions or the integral operator itself.


\section{PROOFS OF DETERMINISTIC-SAMPLE RESULTS}

We begin with the proofs of the deterministic-sample results (Theorems~\ref{thm:reg_bias_deterministic} and~\ref{thm:reg_var_deterministic}).

In this section, we will often abbreviate scaled identity operators such as $a I_n$, $a \scrI$, $a \scrI_G$, $a \scrI_G{^\perp}$ by the number $a$. The meaning should be clear from context.
\label{app:deterministic_proofs}
\subsection{Bias}
\label{app:bias_proof}
The main technical challenge for proving \Cref{thm:reg_bias_deterministic} is bounding the approximation error between the ``ideal'' bias operator $\scrBbr = \parens*{\scrI_G + \frac{n}{\alphabr} \scrT_G}^{-1}$ (discussed in \Cref{sec:proof_sketch})
and the actual bias, which turns out to be $\scrB \coloneqq (\scrI_G + \scrA_G^* (\alpha + \scrR \scrR^*)^{-1} \scrA_G)^{-1}$ (derived in the proof of \Cref{thm:reg_bias_deterministic} below).
The following result quantifies this error.
\begin{lemma}
	\label{lem:bias_approx_error}
	Under the conditions of \Cref{thm:reg_bias_deterministic},
	\[
		\normsub{\scrB - \scrBbr}{\scrH \to L_2} \leq \frac{c}{1-c} \normsub{\scrBbr}{\scrH \to L_2},
	\]
	where $c < 1$ is an upper bound on the quantity $\frac{\alpha_U - \alpha_L}{\alpha_U + \alpha_L} + \frac{2}{n} \norm*{ \scrC^*\scrC - n \scrI_G}_{L_2}$ (as defined in~\Cref{thm:reg_bias_deterministic}).
\end{lemma}
\begin{proof}
	Recall that we have assumed
	\[
		\frac{\alpha_U - \alpha_L}{\alpha_U + \alpha_L} + \frac{2}{n} \norm*{ \scrC^*\scrC - n \scrI_G}_{L_2} \leq c < 1.
	\]
	A standard perturbation argument (e.g., \cite[p.\ 335]{Horn1985}) gives
	\begin{align*}
		\scrB - \scrBbr
		&= \sum_{i = 1}^\infty (-1)^{i} \brackets*{ \scrBbr \parens*{ \scrA_G^* (\alpha + \scrR \scrR^*)^{-1} \scrA_G - \frac{n}{\alphabr} \scrT_G } }^k \scrBbr \\
		&= \parens*{\sum_{i=1}^\infty (-1)^{i} \brackets*{ \parens*{ \scrT_G^{-1} + \frac{n}{\alphabr} \scrI_G }^{-1} \parens*{ \scrC^* (\alpha + \scrR \scrR^*)^{-1} \scrC - \frac{n}{\alphabr} \scrI_G } }^k } \scrBbr
	\end{align*}
	as long as the operator norm (in any space) of the bracketed operator $\parens*{ \scrT_G^{-1} + \frac{n}{\alphabr} \scrI_G }^{-1} \parens*{ \scrC^* (\alpha + \scrR \scrR^*)^{-1} \scrC - \frac{n}{\alphabr} \scrI_G }$ is strictly less than $1$.
	
	We now show that this is the case.
	We have
	\begin{align*}
		\norm*{\scrC^* (\alpha + \scrR \scrR^*)^{-1} \scrC - \frac{n}{\alphabr} \scrI_G}_{L_2}
		&\leq \norm*{ \scrC^* \parens*{ (\alpha + \scrR\scrR^*)^{-1} - \frac{1}{\alphabr} } \scrC }_{L_2} + \frac{1}{\alphabr} \norm*{ \scrC^*\scrC - n \scrI_G}_{L_2} \\
		&\leq \norm{\scrC}_{L_2 \to \ell_2}^2 \max\braces*{ \abs*{\frac{1}{\alpha_L} - \frac{1}{\alphabr}}, \abs*{\frac{1}{\alpha_U} - \frac{1}{\alphabr}} } + \frac{1}{\alphabr} \norm*{ \scrC^*\scrC - n \scrI_G}_{L_2} \\
		&\leq \frac{\alpha_U - \alpha_L}{2 \alpha_U \alpha_L} \parens*{ n + \norm*{ \scrC^*\scrC - n \scrI_G}_{L_2} } + \frac{1}{\alphabr} \norm*{ \scrC^*\scrC - n \scrI_G}_{L_2},
	\end{align*}
	where the first and third inequalities use the triangle inequality, and the second inequality uses $\alpha_L \preceq \alpha + \scrR \scrR^* \preceq \alpha_U$.
	Then, since
	\[
		\norm*{ \parens*{\scrT_G^{-1} + \frac{n}{\alphabr} \scrI_G}^{-1} }_{L_2} \leq \frac{\alphabr}{n},
	\]
	we have
	\begin{align*}
		&\norm*{ \parens*{ \scrT_G^{-1} + \frac{n}{\alphabr} \scrI_G }^{-1} \parens*{ \scrC^* (\alpha + \scrR \scrR^*)^{-1} \scrC - \frac{n}{\alphabr} \scrI_G } }_{L_2} \\
		&\qquad\leq \frac{\alpha_U - \alpha_L}{\alpha_U + \alpha_L} + \parens*{ 1 + \frac{\alpha_U - \alpha_L}{\alpha_U + \alpha_L} } \cdot \frac{1}{n} \norm*{ \scrC^*\scrC - n \scrI_G}_{L_2} \\
		&\qquad\leq \frac{\alpha_U - \alpha_L}{\alpha_U + \alpha_L} + \frac{2}{n} \norm*{ \scrC^*\scrC - n \scrI_G}_{L_2} \\
		&\qquad\leq c.
	\end{align*}
	Since $c < 1$, the rest of the bound follows via the expression for the infinite sum of a geometric series.
\end{proof}

We are now ready to prove our main deterministic bias result (\Cref{thm:reg_bias_deterministic}).
\begin{proof}[Proof of \Cref{thm:reg_bias_deterministic}]
	Since $f^* \in G$,
	the full expression for the noiseless regression estimate is
	\[
		\fhat_0 = \begin{bmatrix*} \scrA_G^* \\ \scrR^* \end{bmatrix*} (\alpha + \scrA_G \scrA_G^* + \scrR \scrR^*)^{-1} \scrA_G f^*.
	\]
	The pushthrough identity gives
	\begin{align*}
		(\alpha + \scrA_G \scrA_G^* + \scrR \scrR^*)^{-1} \scrA_G
		&= (\alpha + \scrR \scrR^*)^{-1} (I_n + \scrA_G \scrA_G^* (\alpha + \scrR \scrR^*)^{-1})^{-1} \scrA_G \\
		&= (\alpha + \scrR \scrR^*)^{-1} \scrA_G (\scrI_G + \scrA_G^* (\alpha + \scrR \scrR^*)^{-1} \scrA_G)^{-1}.
	\end{align*}
	This gives
	\begin{align*}
		\scrP_G(\fhat_0)
		&= \scrA_G^* (\alpha + \scrR \scrR^*)^{-1} \scrA_G (\scrI_G + \scrA_G^* (\alpha + \scrR \scrR^*)^{-1} \scrA_G)^{-1} f^* \\
		&= (\scrI_G - (\scrI_G + \scrA_G^* (\alpha + \scrR \scrR^*)^{-1} \scrA_G)^{-1}) f^*
	\end{align*}
	and
	\begin{align*}
		\scrP_{G^\perp}(\fhat_0)
		&= \scrR^* (\alpha + \scrR \scrR^*)^{-1} \scrA_G (\scrI_G + \scrA^* (\alpha + \scrR \scrR^*)^{-1} \scrA^*)^{-1} f^*.
	\end{align*}
	We denote the actual bias and survival operators on $G$ as
	 \begin{align*}
	     \scrB &= (\scrI_G + \scrA_G^* (\alpha + \scrR \scrR^*)^{-1} \scrA_G)^{-1}, \text{ and } \\
	     \scrS &= \scrI_G - \scrB.
	 \end{align*}
	We then have
	\[
		\scrP_G(\fhat_0) = \scrS f^*,
	\]
	and
	\[
		\scrP_{G^\perp}(\fhat_0) = \scrR^* (\alpha + \scrR \scrR^*)^{-1} \scrA_G \scrB f^*.
	\]
	
	Clearly, $\norm{f^* - \scrP_G (\fhat_0)}_{L_2} = \norm{\scrB f^*}_{L_2} \leq \normsub{\scrB}{\scrH \to L_2} \normH{f^*}$.
	To bound $\norm{\scrP_{G^\perp}(\fhat_0)}_{L_2}$, note that (recalling $\scrC = \scrA_G$)
	\[
		\normsub*{\scrR^* (\alpha + \scrR \scrR^*)^{-1} \scrA_G}{L_2}
		\leq \normsub{\scrI_{G^\perp}}{\scrH \to L_2} \cdot \normsub{ \scrR^* (\alpha + \scrR \scrR^*)^{-1} }{\ell_2 \to \scrH} \cdot \normsub{\scrC}{L_2 \to \ell_2}.
	\]
	Note that $\norm{\scrI_{G^\perp}}_{\scrH \to L_2} = \normH{\scrT_{G^\perp}^{1/2}} = \sqrt{\lambda_{p+1}}$,
	and
	\[
		\norm{\scrC}_{L_2 \to \ell_2}^2 = \norm{\scrC^* \scrC}_{L_2} \approx \norm{n \scrI_G}_{L_2} = n.
	\]
	Furthermore, note that the singular values (from $\ell_2$ to $\scrH$) of the operator $\scrR^* (\alpha + \scrR \scrR^*)^{-1}$
	are 
	\[
		\frac{\sqrt{\lambda_k(\scrR \scrR^*)}}{\alpha + \lambda_k(\scrR \scrR^*)} \leq \frac{1}{\sqrt{\alpha + \lambda_k(\scrR \scrR^*)}} \leq \frac{1}{\sqrt{\alpha_L}},~k = 1,\dots, n,
	\]
	where $\lambda_k(S)$ denotes the $k$th eigenvalue of a symmetric matrix $S$.
	Therefore,
	\[
		\normsub*{\scrR^* (\alpha + \scrR \scrR^*)^{-1} \scrA_G}{L_2}
		\lesssim \sqrt{\lambda_{p+1}} \cdot \frac{1}{\sqrt{\alpha_L}} \cdot \sqrt{n} = \sqrt{\frac{n \lambda_{p+1}}{\alpha_L}}.
	\]
	
    Noting that $\alphabr \leq 2 \alpha_L$, we have
	\[
		\norm{\fhat_0 - f^*}_{L_2} \lesssim \parens*{ 1 + \sqrt{\frac{n \lambda_{p+1}}{\alphabr}} } \normsub{\scrB}{\scrH \to L_2} \normH{f^*}.
	\]
	
	\Cref{lem:bias_approx_error} gives
	\[
		\normsub{\scrB}{\scrH \to L_2} \leq \normsub{\scrBbr}{\scrH \to L_2} + \normsub{\scrB - \scrBbr}{\scrH \to L_2}
		\leq \frac{1}{1-c} \normsub{\scrBbr}{\scrH \to L_2}.
	\]
	Also, one can also easily check that $\normH{B} \leq 1$,
	and therefore $\normsub{\scrB}{\scrH \to L_2} \leq \sqrt{\lambda_1}$.
	Thus
	\[
		\normsub{\scrB}{\scrH \to L_2} \leq \min\braces*{ \sqrt{\lambda_1}, \frac{1}{1-c} \normsub{\scrBbr}{\scrH \to L_2} }
	\]
	Using the fact that $\normsub{\scrBbr}{\scrH \to L_2} \lesssim \min\braces*{ \sqrt{\frac{\alphabr}{n}}, \frac{\alphabr}{n \sqrt{\lambda_p} }}$ completes the proof.
	
\end{proof}
With the proof of \Cref{thm:reg_bias_deterministic} complete, recall that we introduced a more refined expression for the estimation error due to bias in~\Cref{lem:approx_error} for the purpose of bounding classification error.
Note that the proof of \Cref{lem:approx_error} is a very simple modification of the preceding proof.
The error in $G^\perp$ is bounded the same way.
For the error in $G$, we bound the norm of $(\scrS - \scrSbr) f^* = (\scrBbr - \scrB) f^*$ instead of $f^* - \scrS f^* = \scrB f^*$,
and therefore we replace $\normsub{\scrB}{\scrH \to L_2}$ by $\normsub{\scrBbr - \scrB}{\scrH \to L_2}$ in the bound.

\subsection{Variance}
\label{app:variance_proof}
Recall that $\alpha_L \preceq \alpha + \scrR \scrR^* \preceq \alpha_U$ and $\alphatl = \frac{\alpha_U + \alpha_L}{2}$.
Also recall the formula
\[
	\epsilon = \scrA^* (\alpha + \scrA \scrA^*)^{-1} \xi.
\]
To allow us to replace $\alpha + \scrA \scrA^*$ with $\alphatl + \scrA_G \scrA_G^*$, we need the following result:
\begin{lemma}
	\label{lem:Gram_error}
	\[
		\norm{(\alphatl + \scrA_G \scrA_G^*) (\alpha + \scrA \scrA^*)^{-1}}_{\ell_2} \leq \frac{1}{2} \parens*{ \frac{\alpha_U}{\alpha_L} + 1 }.
	\]
\end{lemma}
\begin{proof}
	Since $(\alpha + \scrA \scrA^*) - (\alphatl + \scrA_G \scrA_G^*) = \alpha + \scrR \scrR^* - \alphatl$,
    another perturbation expansion (see \Cref{app:bias_proof}) gives
	\begin{align*}
		(\alphatl + \scrA_G \scrA_G^*)^{-1} - (\alpha + \scrA \scrA^*)^{-1}
		&= \sum_{k=1}^\infty (-1)^{k+1} (\alphatl + \scrA_G \scrA_G^*)^{-1} \brackets*{ (\alpha + \scrR \scrR^* - \alphatl) (\alphatl + \scrA_G \scrA_G^*)^{-1} }^k,
	\end{align*}
	which is valid since $\alpha_L \preceq \alpha + \scrR \scrR^* \preceq \alpha_U$ implies
	\[
		\norm*{(\alpha + \scrR \scrR^* - \alphatl) (\alphatl + \scrA_G \scrA_G^*)^{-1}}_{\ell_2}
		\leq \frac{1}{\alphatl} \cdot \frac{\alpha_U - \alpha_L}{2}
		= \frac{\alpha_U - \alpha_L}{\alpha_U + \alpha_L} < 1.
	\]
	Then
	\begin{align*}
		I_n - (\alphatl + \scrA_G \scrA_G^*) (\alpha + \scrA \scrA^*)^{-1}
		&= \sum_{k=1}^\infty (-1)^{k+1} \brackets*{ (\alpha + \scrR \scrR^* - \alphatl) (\alphatl + \scrA_G \scrA_G^*)^{-1} }^k.
	\end{align*}
	
	We apply the triangle inequality to get
	\begin{align*}
		\normsub{(\alphatl + \scrA_G \scrA_G^*) (\alpha + \scrA \scrA^*)^{-1}}{\ell_2}
		&\leq \normsub{I_n}{\ell_2} + \sum_{k=1}^\infty \norm*{ (\alpha + \scrR \scrR^* - \alphatl) (\alphatl + \scrA_G \scrA_G^*)^{-1} }_{\ell_2}^k \\
		&\leq 1 + \sum_{i=1}^\infty \parens*{ \frac{\alpha_U - \alpha_L}{\alpha_U + \alpha_L}}^i \\
		&= \frac{1}{2}\parens*{ \frac{\alpha_U}{\alpha_L} + 1 }.
	\end{align*}
	
\end{proof}

We can now prove the main ``variance'' error bound:
\begin{proof}[Proof of \Cref{thm:reg_var_deterministic}]
	Since $\var(\xi_i) \leq \sigma^2$ for each $i$, we have
	\begin{align*}
		\E_\xi \norm{\epsilon}_{L_2}^2
		&\leq \sigma^2 \norm{\scrA^* (\alpha + \scrA \scrA^*)^{-1}}_{HS, \ell_2 \to L_2}^2 \\
		&= \sigma^2 \norm{\scrA^* (\alphatl + \scrA_G \scrA_G^*)^{-1} (\alphatl + \scrA_G \scrA_G^*) (\alpha + \scrA \scrA^*)^{-1}}_{HS, \ell_2 \to L_2}^2 \\
		&\leq \frac{\sigma^2}{4} \parens*{\frac{\alpha_U}{\alpha_L} + 1}^2  \norm{\scrA^* (\alphatl + \scrA_G \scrA_G^*)^{-1}}_{HS, \ell_2 \to L_2}^2,
	\end{align*}
	where the last inequality substitutes~\Cref{lem:Gram_error}.
	Furthermore, we have
	\begin{align*}
		\norm{\scrA^* (\alphatl + \scrA_G \scrA_G^*)^{-1}}_{HS, \ell_2 \to L_2}^2
		&= \norm{\scrA_G^* (\alphatl + \scrA_G \scrA_G^*)^{-1}}_{HS, \ell_2 \to L_2}^2 + \norm{\scrR^* (\alphatl + \scrA_G \scrA_G^*)^{-1}}_{HS, \ell_2 \to L_2}^2 \\
		&\leq \norm{(\alphatl + \scrA_G^* \scrA_G)^{-1} \scrA_G^* }_{HS, \ell_2 \to L_2}^2 + \frac{\tr_{L_2}(\scrR^* \scrR)}{\alphatl^2} \\
		&= \norm{(\alphatl \scrT_G^{-1}  + \scrC^* \scrC)^{-1} \scrC^* }_{HS, \ell_2 \to L_2}^2 + \frac{\tr_{L_2}(\scrR^* \scrR)}{\alphatl^2} \\
		&\lesssim \frac{p}{n} + \frac{\tr_{L_2}(\scrR^* \scrR)}{\alphatl^2},
	\end{align*}
	where the last inequality is due to the fact that $\scrC$ is an $n \times p$-dimensional operator, all of whose singular values are close to $\sqrt{n}$.
	
	Therefore,
	\[
		\E_\xi \norm{\epsilon}_{L_2}^2 \lesssim \sigma^2 \parens*{\frac{\alpha_U}{\alpha_L} + 1}^2 \parens*{ \frac{p}{n} + \frac{\tr_{L_2}(\scrR^* \scrR)}{\alphatl^2} }.
	\]
\end{proof}

\subsubsection{High-probability Noise Bounds}
If the $\xi_i$'s are sub-Gaussian,
we could use the Hanson-Wright inequality for sub-Gaussian random vectors (see, e.g., \cite{Rudelson2013})
to get a high-probability bound in \Cref{thm:reg_var_deterministic},

Note that we can write
\[
	\norm{\epsilon}_{L_2}^2 = \ip*{ Z \xi }{\xi}_{\ell_2},
\]
where
\[
Z = (\alpha + \scrA \scrA^*)^{-1}\scrA \scrT \scrA^* (\alpha + \scrA \scrA^*)^{-1}.
\]
We have already calculated an upper bound on the expectation of this quadratic form.
To use the Hanson-Wright inequality to bound the upper tail, we need to bound both $\norm{Z}_{\ell_2}$ and $\normHS{Z}$
(where $\normHS{Z}$ is the Hilbert-Schmidt norm with respect to the Euclidean inner product, also known as the Frobenius norm).
By a similar argument as before, we have
\[
\norm{Z}_{\ell_2} \leq \frac{1}{4} \parens*{ \frac{\alpha_U}{\alpha_L} + 1 }^2 \norm {\Ztl}_{\ell_2},
\]
and
\[
\normHS{Z} \leq \frac{1}{4} \parens*{ \frac{\alpha_U}{\alpha_L} + 1 }^2 \normHS{\Ztl},
\]
where
\begin{align*}
	\Ztl &= (\alphatl + \scrA_G \scrA_G^*)^{-1} \scrA \scrT \scrA^* (\alphatl + \scrA_G \scrA_G^*)^{-1} \\
	&= \underbrace{(\alphatl + \scrA_G \scrA_G^*)^{-1} \scrA_G \scrT_G \scrA_G^* (\alphatl + \scrA_G \scrA_G^*)^{-1}}_{\Ztl_G} + \underbrace{(\alphatl + \scrA_G \scrA_G^*)^{-1} \scrR \scrT_{G^\perp} \scrR^* (\alphatl + \scrA_G \scrA_G^*)^{-1}}_{\Ztl_{G^\perp}}.
\end{align*}
Note that
\[
\Ztl_G  = \scrA_G (\alphatl + \scrA_G^* \scrA_G)^{-1} \scrT_G (\alphatl + \scrA_G^* \scrA_G)^{-1} \scrA_G^*
= \scrC (\alphatl \scrT_G^{-1} + \scrC^* \scrC)^{-2} \scrC^*.
\]
By a similar argument as before (in which we were effectively calculating the trace of $\Ztl_G$),
we have $\norm{\Ztl_G}_{\ell_2} \lesssim \frac{1}{n}$ and $\normHS{\Ztl_G} \lesssim \frac{\sqrt{p}}{n}$.

Similarly, $\norm{\Ztl_{G^\perp}}_{\ell_2} \leq \frac{1}{\alphatl^2} \norm{\scrR \scrT_{G^\perp} \scrR^*}_{\ell_2}$, and $\normHS{\Ztl_{G^\perp}} \leq \frac{1}{\alphatl^2} \normHS{\scrR \scrT_{G^\perp} \scrR^*}$.

\section{PROOFS OF OPERATOR CONCENTRATION RESULTS}
\label{app:conc_proofs}
\begin{proof}[Proof of \Cref{lem:R_conc_generic}]
	Let $\diag(Z)$ denote the projection of $Z$ onto the space of diagonal matrices,
	and let $\diag^\perp(Z)$ denote the orthogonal projection (i.e., onto the space of matrices with zero diagonal).
	Note that
	\begin{align*}
		\norm{\scrR \scrR^* - (\tr \scrT_{G^\perp}) I_n}
		&\leq \norm{\diag^\perp(\scrR \scrR^*)} + \norm{\diag(\scrR \scrR^*) - (\tr \scrT_{G^\perp}) I_n} \\
		&\leq \normHS{\diag^\perp(\scrR \scrR^*)} + \max_i~\abs{k^R(x_i, x_i) - \tr \scrT_{G^\perp}} \\
		&\leq \sqrt{\sum_{i \neq j} (k^R(x_i, x_j))^2} + \norm{k^R(\cdot, \cdot) - \tr \scrT_{G^\perp}}_{\infty}. \\
	\end{align*}
	Squaring, taking an expectation, and noting that
	\[
		\E_{x,y \simiid \mu} (k^R(x, y))^2 = \tr (\scrT_{G^\perp}^2)
	\]
	completes the proof.
\end{proof}

\begin{proof}[Proof of \Cref{lem:R_trace_exp}]
	We have
	\begin{align*}
		\tr_{L_2}(\scrR^* \scrR) 
		&= \sum_{i=1}^n \tr_{L_2}(k^R_{x_i} \otimes k^R_{x_i}) \\
		&= \sum_{i=1}^n \norm{k^R_{x_i}}_{L_2}^2 \\
		&= \sum_{i=1}^n \sum_{\ell > p} \lambda_\ell^2 v_\ell^2(x_i).
	\end{align*}
	Taking an expectation completes the proof.
\end{proof}

\begin{proof}[Proof of \Cref{lem:C_conc_BOS}]
    We can write the operator $\scrC^* \scrC$ as a sum of independent random operators:
    \[
        \scrC^* \scrC = \sum_{i=1}^n z(x_i) \otimes z(x_i),
    \]
    where
    \[
        z(x) \coloneqq \sum_{\ell = 1}^p v_\ell(x) v_\ell.
    \]
    Note that the BOS condition implies $\norm{z(x)}_{L_2}^2 \leq Cp$ almost surely in $x$.
    We also have $\E z(x) \otimes z(x) = \scrI_G$ for $x \sim \mu$.
    
    We use a matrix Bernstein inequality \parencite[Theorem 6.6.1]{Tropp2015}
    to analyze the zero-mean sum
    \[
        \scrC^* \scrC - n \scrI_G = \sum_{i=1}^p (z(x_i) \otimes z(x_i) - \E z(x_i) \otimes z(x_i)).
    \]
    Writing $X_i = z(x_i) \otimes z(x_i) - \E z(x_i) \otimes z(x_i)$,
    we have $\norm{X_i}_{L_2} \leq Cp$ almost surely,
    and
    \[
        \E X_i^2
        \preceq \E (z(x_i) \otimes z(x_i))^2
        = \E \norm{z(x_i)}_{L_2}^2 z(x_i) \otimes z(x_i)
        \preceq C p \E z(x_i) \otimes z(x_i)
        = Cp \scrI_G.
    \]
    The Bernstein inequality then gives that for any $t > 0$,
    with probability at least $1 - e^{-t}$,
    \[
        \norm{\scrC^* \scrC - n \scrI_G}_{L_2}
        = \norm*{\sum_{i=1}^n X_i}_{L_2}
        \lesssim \sqrt{C p n (t + \log p)} + Cp (t + \log p).
    \]
\end{proof}

	

\begin{proof}[Proof of \Cref{lem:R_conc_ind}]
	For $z \in \R^n$,
	we have
	\[
		\ip{\scrR \scrR^* z}{z} = \sum_{\ell > p}\lambda_\ell \ip{w_\ell}{z}^2.
	\]
	By our assumptions, this is the sum of independent random variables.
	
	If $\norm{z}_{\ell_2} = 1$,
	then, for each $\ell$, $\ip{w_\ell}{z}^2$ is sub-exponential (as the square of a sub-Gaussian variable; since the sub-Gaussian norm is bounded, so is the sub-exponential norm), $\E \ip{w_\ell}{z}^2 = 1$, and $\E \ip{w_\ell}{z}^4 \lesssim 1$.
	
	Note that in this case, $\E \ip{\scrR \scrR^* z}{z} = \tr \scrT_{G^\perp} = \ip{(\tr \scrT_{G^\perp}) I_n z}{z}$,
	and
	\begin{align*}
		\E (\ip{\scrR \scrR^* z}{z} - \E \ip{\scrR \scrR^* z}{z})^2
		&= \sum_{\ell > p} \lambda_\ell^2 \E (\ip{w_\ell}{z}^2 - \E \ip{w_\ell}{z}^2)^2 \\
		&\lesssim \sum_{\ell > p} \lambda_\ell^2.
	\end{align*}
	A Bernstein inequality then implies that for $t > 0$, with probability at least $1 - e^{-t}$, we have
	\[
		\abs{  \ip{\scrR \scrR^* z}{z} - \tr \scrT_{G^\perp} } \lesssim \sqrt{ \parens*{ \sum_{\ell > p} \lambda_\ell^2 } t} + \lambda_{p+1} t.
	\]
	By a standard covering argument (e.g., \cite[Exercise 4.4.3]{Vershynin2018}),
	we then obtain, with probability at least $1 - e^{-t}$,
	\[
		\max_{z \in S^{{n-1}}} \abs{  \ip{\scrR \scrR^* z}{z} - \tr \scrT_{G^\perp} }
		\lesssim \sqrt{ \parens*{ \sum_{\ell > p} \lambda_\ell^2 } (n + t)} + \lambda_{p+1} (n+t),
	\]
	where $S^{n-1}$ is the unit sphere in $\R^n$.
\end{proof}

\section{TIGHTNESS OF GENERAL FEATURE RESULTS}
\label{sec:FourierExample}
With no independence assumptions on the features $\{v_{\ell}(x)\}_{\ell}$, our general results require $d \gtrsim n^2$ in order to upper and lower bound the residual Gram matrix $\scrR \scrR^*$ by constant multiples of the identity. The following theorem shows that for Fourier features, $d \gtrsim n^2$ is in fact a necessary condition, i.e. if $d = o(n^2)$, then the condition number of $\scrR\scrR^*$ grows as $n \to \infty$. 

\begin{theorem}
Consider the case of Fourier features with bi-level eigenvalues, i.e. $v_{\ell} \in L_2([0,1])$ for $\ell = -d,\ldots,d$, which are defined by $v_{\ell}(x) = e^{\mathbf{j}2\pi \ell x}$ for $x \in [0,1]$, and $\lambda_{\ell} = 1$ for $|\ell| \le p$, $\lambda_{\ell} = \gamma \in (0,1)$ for $p < |\ell| \le d$.  Then, for any constant $\tau > 0$, the residual Gram matrix $\scrR\scrR^*$ satisfies $$\dfrac{\lambda_{\text{max}}(\scrR\scrR^*)}{\lambda_{\text{min}}(\scrR\scrR^*)} \gtrsim \dfrac{n^4}{\tau^2d^2}$$ with probability at least $1-e^{-\tau}$. 
\end{theorem}

Intuitively, if there exist distinct indices $i, i' = 1,\ldots,n$ such that $x_i$ and $x_{i'}$ are very close together, then the $i$-th and $i'$-th columns (and rows) of $\scrR\scrR^*$ are nearly identical, and thus, $\scrR\scrR^*$ is nearly rank-deficient. We now make this argument rigorous.

\begin{proof}
First, pick any two indices $i,i' \in \{1,\ldots,n\}$ with $i \neq i'$ and consider the $2\times 2$ submatrix of $\scrR\scrR^*$ formed by the $i$-th and $i'$-th rows and columns, i.e, $$(\scrR\scrR^*)_{\text{sub}} := \begin{bmatrix}k^R(x_i,x_i) & k^R(x_i,x_{i'}) \\ k^R(x_{i'},x_i) & k^R(x_{i'},x_{i'}) \end{bmatrix}.$$ 

The kernel restricted to $G^{\perp}$ is given by \begin{align*}k^R(x,y) &= \sum_{p < |\ell| \le d}\lambda_{\ell}v_{\ell}(x)\overline{v_{\ell}(y)} 
\\
&= \sum_{p < |\ell| \le d}\gamma e^{\mathbf{j}2\pi\ell(x-y)}
\\
&= \gamma\dfrac{\sin[(2d+1)\pi(x-y)]-\sin[(2p+1)\pi(x-y)]}{\sin[\pi(x-y)]}.\end{align*} Hence, $k^R(x_i,x_i) = k^R(x_{i'},x_{i'}) = 2(d-p)\gamma$. 

Furthermore, using the inequality $2\cos\theta \ge 2-\theta^2$ for $\theta \in \R$, we have \begin{align*}
\dfrac{\sin[(2d+1)\pi t]-\sin[(2p+1)\pi t]}{\sin[\pi t]} &= \sum_{p < |\ell| \le d}e^{\mathbf{j}2\pi\ell t} 
\\
&= \sum_{\ell = p+1}^{d}2\cos(2\pi\ell t) 
\\
&\ge \sum_{\ell = p+1}^{d}\left[2-(2\pi\ell t)^2\right] 
\\
&= 2(d-p)-4\pi^2\left(\sum_{\ell = p+1}^{d}\ell^2\right)t^2
\\
&\ge 2(d-p)-4\pi^2d^2(d-p)t^2
\end{align*} for all $t \in \R$, and thus, \begin{align*}
k^R(x_i,x_{i'}) = k^R(x_{i'},x_i) &= \gamma\dfrac{\sin[(2d+1)\pi(x_i-x_{i'})]-\sin[(2p+1)\pi(x_i-x_{i'})]}{\sin[\pi(x_i-x_{i'})]} 
\\
&\ge 2(d-p)\gamma - 4\pi^2d^2(d-p)\gamma(x_i-x_{i'})^2.
\end{align*} We can then bound the smallest eigenvalue of $\scrR\scrR^*$ by  $$\lambda_{\text{min}}(\scrR\scrR^*) \le \lambda_{\text{min}}((\scrR\scrR^*)_{\text{sub}}) = k^R(x_i,x_{i}) - k^R(x_i,x_{i'}) \le 4\pi^2d^2(d-p)\gamma(x_i-x_{i'})^2.$$ Then, by using the trivial bound $\lambda_{\text{max}}(\scrR\scrR^*) \ge \tfrac{1}{n}\tr(\scrR\scrR^*) = \tfrac{1}{n} \cdot 2(d-p)\gamma n = 2(d-p)\gamma$, we have $$\dfrac{\lambda_{\text{max}}(\scrR\scrR^*)}{\lambda_{\text{min}}(\scrR\scrR^*)} \ge \dfrac{2(d-p)\gamma}{4\pi^2d^2(d-p)\gamma(x_i-x_{i'})^2} = \dfrac{1}{2\pi^2d^2(x_i-x_{i'})^2}.$$

This bound holds for any distinct indices $i \neq i'$. A relatively straightforward calculation\footnote{Thanks to Hans's answer at \url{https://mathoverflow.net/questions/1294}} shows that if $x_1,\ldots,x_n$ are i.i.d. $\text{Uniform}[0,1]$, then for any $\delta \in (0,\tfrac{1}{n-1})$, \begin{align*}
\P\left\{ |x_i-x_{i'}| \ge \delta \ \text{for all} \ i \neq i'\right\} &= n! \P\left\{x_{i-1}+\delta \le x_i \ \text{for all} \ i = 2,\ldots,n\right\}
\\
&= n! \int \cdots \int_{\{0 \le x_1, \ x_{i-1} + \delta \le x_i \ \text{for} \ i = 2,\ldots, n, \ x_n \le 1\}} \,dx_1\cdots\,dx_n
\\
&= n! \int \cdots \int_{\{y_i \ge 0 \ \text{for} \ i = 1,\ldots,n, \ y_1+\cdots+y_n \le 1-(n-1)\delta\}} \,dy_1\cdots\,dy_n
\\
&= n! \cdot \dfrac{1}{n!}(1-(n-1)\delta)^n
\\
&= (1-(n-1)\delta)^n
\end{align*} where we made the change of variable $y_1 = x_1$ and $y_i = x_i-x_{i-1}-\delta$ for $i = 2,\ldots,n$, and we used the fact that the volume of the standard $n$-simplex is $\tfrac{1}{n!}$. Hence, if $0 < \delta < \tfrac{1}{n-1}$, the probability that $|x_i-x_{i'}| \le \delta$ for some indices $i \neq i'$ is $1-(1-(n-1)\delta)^n$.

If $0 < \tau < n$, we can apply this result for $\delta = \tfrac{\tau}{n(n-1)}$, to obtain that with probability $1-(1-\tfrac{\tau}{n})^n \ge 1-e^{-\tau}$ there exist $i \neq i'$ such that $|x_i-x_{i'}| \le \tfrac{\tau}{n(n-1)}$, and thus, $$\dfrac{\lambda_{\text{max}}(\scrR\scrR^*)}{\lambda_{\text{min}}(\scrR\scrR^*)} \ge \dfrac{1}{2\pi^2d^2(x_i-x_{i'})^2} \ge \dfrac{n^2(n-1)^2}{2\pi^2d^2\tau^2} \gtrsim \dfrac{n^4}{\tau^2d^2}.$$ If $\tau \ge n$, then it is guaranteed that there exist two indices $i \neq i'$ which satisfy $|x_i-x_{i'}| \le \tfrac{1}{n-1} \le \tfrac{\tau}{n(n-1)}$, and the same bound holds.
\end{proof}

\section{PROOF OF BI-LEVEL ENSEMBLE ASYMPTOTIC RESULTS}
\label{app:bilevel_proof}

If $\beta > 2$ and $r < 1$, the concentration results \Cref{lem:C_conc_BOS,lem:R_conc_generic} will hold as $n$ becomes large,
since we will have $n \gg p \log p$ and $d - p \approx n^\beta \gg n^2$.
We now apply~\Cref{lem:approx_error} and~\Cref{thm:reg_var_deterministic} to the bi-level ensemble.
Since we are in the interpolating regime, we take $\alpha = 0$; then, $\alpha_L = \lambda_{\min}(\scrR \scrR^*)$ and $\alpha_U = \lambda_{\max}(\scrR \scrR^*)$ are the smallest and largest eigenvalues of $\scrR \scrR^*$.
As long as $\alpha_L$ and $\alpha_U$ are close together (which we will analyze next),
we will have
\[
	\alphatl \approx \alphabr \approx \sum_{\ell > p} \lambda_\ell \approx n^\beta \cdot n^{-(\beta - r - q)} = n^{r + q}.
\]
Furthermore,
\[
    \sum_{\ell > p} \lambda_\ell^2 \approx n^\beta n^{-2(\beta - q - r)} = n^{2q + 2r -\beta}.
\]
Applying these scalings to~\Cref{thm:reg_var_deterministic} gives us
\[
\E_\xi \norm{\epsilon}_{L_2}^2 \lesssim n^{r - 1} + \frac{n}{n^{2(r + q)}} n^{2q + 2r - \beta} = n^{r-1} + n^{1-\beta}.
\]

To bound the bias,
note that combining the above calculations with \Cref{lem:R_conc_generic} gives
\begin{align*}
	\frac{\alpha_U - \alpha_L}{\alpha_U + \alpha_L}
	&\lesssim \frac{1}{\alphabr} \sqrt{n^2 \sum_{\ell > p} \lambda_\ell^2} \\
	&\approx \frac{1}{n^{r+q}} \sqrt{n^2 n^{2q + 2r - \beta}} \\
	&= n^{1 - \beta/2}.
\end{align*}
Combining this with \Cref{lem:C_conc_BOS},
the quantity $c$ in \Cref{thm:reg_bias_deterministic,lem:approx_error} can be bounded as
\[
    c \lesssim n^{1 - \beta/2} + n^{(r-1)/2} \sqrt{\log n}.
\]
Then \Cref{lem:approx_error} gives
\begin{align*}
    \frac{\norm{\etahat_0 - \scrSbr \eta^*}_{L_2}}{\normH{\eta^*}}
    &\lesssim \parens*{ n^{1 - \beta/2} + n^{(r-1)/2} \sqrt{\log n} + \sqrt{\frac{n^{-(\beta - r - q)} n}{n^{r + q}}}}\cdot \min\braces*{ 1,  n^{r + q - 1}, n^{(r + q - 1)/2} } \\
    &\lesssim \parens*{ n^{1 - \beta/2} + n^{(r-1)/2} \sqrt{\log n}} \min\{1, n^{r + q - 1}\}.
\end{align*}

Recall from \eqref{eq:riskboundratio} that excess classification risk has upper bound $\scrE \leq \frac{\norm{\etahat_r}_{L_2}}{s}$ for any decomposition $\etahat_0 = s \eta^* + \etahat_r$ with an $s > 0$ that we can choose.
We will now characterize the terms $s$ and $\norm{\etahat_r}_{L_2}$, beginning with the factor $s$.
The ideal survival operator is given by
\[
\scrSbr
= \scrI_G - \parens*{ \scrI_G + \frac{n}{\alphabr} \scrT_G }^{-1}
= \frac{1}{1 + \alphabr n} \scrI_G \approx \frac{1}{1 + n^{r+q-1}} \scrI_G.
\]

Then, we can decompose
\[
\etahat = \scrSbr \eta^* + \etahat_r \approx \frac{1}{1 + n^{r+q - 1}} \eta^* + \etahat_r,
\]
where $\etahat_r = \epsilon + \etahat_0 - \scrSbr \eta^*$.
This gives us $s \approx \frac{1}{1 + n^{r+q - 1}}$.

Next, we bound $\norm{\etahat_r}_{L_2}$.
We have
\begin{align*}
	\norm{\etahat_r}_{L_2}
	&\lesssim n^{(r-1)/2} + n^{(1-\beta)/2} + \parens*{ n^{1 - \beta/2} + n^{(r-1)/2} \sqrt{\log n}} \cdot \min\{1, n^{r + q - 1}\} \norm{\eta^*}_{L_2}.
\end{align*}
Above, we used the fact that $\norm{\eta^*}_{L_2} = \norm{\eta^*}_{\scrH}$.

There are several cases to consider (recall that we are already assuming $\beta > 2$ and $r < 1$):
\begin{enumerate}
	\item $q < 1 - r$: In this case, $s \approx \frac{1}{1 + n^{r+q - 1}} \to 1$,
	and $\norm{\etahat_r}_{L_2} \to 0$.
	Thus both the excess regression and classification risk converge to $0$ as $n \to \infty$.
	
	\item If $q > 1-r$, we have $s \to 0$ and $\norm{\etahat_r}_{L_2} \to 0$,
	so $\norm{\etahat}_{L_2} \to 0$.
	Therefore will will \emph{not} get regression consistency (for nonzero $\eta^*$).
	
	\item If $1-r < q < \frac{3}{2}(1-r)$ and $\beta > 2r + 2q$, 
	then $s \to 0$,
	but $\frac{\norm{\etahat_r}_{L_2}}{s} \approx \norm{\etahat_r}_{L_2} \cdot (1 + n^{r+q - 1}) \to 0$,
	so the excess classification risk converges to zero as $n \to \infty$ even though the regression risk does not.
	
	\item If $1-r < q < \frac{3}{2}(1-r)$ but $\beta < 2r + 2q$ or if $q > \frac{3}{2}(1-r)$, our analysis does not yield any convergence results.
	It is an interesting and important direction for future work to characterize precisely what relations between the parameters $q, r, \beta$ are both sufficient and necessary for classification risk to go to $0$ as $n \to \infty$.
\end{enumerate}

\section{DISTORTION ANALYSIS}
\label{sec:distortion}
In this section,
we analyze more carefully the regularization-induced distortion.
In particular, we consider how different the (deterministic) ideal survival operator $\scrSbr$ is from a multiple of the identity.
Recall that
\[
\scrSbr = \scrI_G - \parens*{\scrI_G + \frac{n}{\alphabr} \scrT_G}^{-1}
= \frac{n}{\alphabr} \scrT_G \parens*{\scrI_G + \frac{n}{\alphabr} \scrT_G}^{-1}.
\]
We want to solve
\begin{align*}
	\argmin_{s>0}~\norm*{s \scrI_G - \scrSbr}_{\scrH \to L_2}
	&= \argmin_{s>0}~\norm*{ s \scrT_G^{1/2} - \scrT_G^{1/2} \scrSbr}_{L_2} \\
	&= \argmin_{s > 0}~\max_{1 \leq \ell \leq p}~\sqrt{\lambda_\ell} \abs*{s - \frac{\lambda_\ell}{\lambda_\ell + \frac{\alphabr}{n} }}.
\end{align*}
We abbreviate $b \coloneqq \frac{\alphabr}{n}$.
The objective function in $s$ is convex as the maximum of convex functions.
Some convex analysis tell us that there must be (at least) two distinct $i,j \in \{1, \dots, p\}$
such that, for $s$ at its optimal value $s^*$,
both $i$ and $j$ achieve the maximum over $\ell$, and the arguments to the absolute value have different signs.
Assuming, without loss of generality, that $\lambda_j > \lambda_i$,
this implies
\begin{align*}
	\argmin_{s > 0}~\max_{1 \leq \ell \leq p}~\sqrt{\lambda_\ell} \abs*{s - \frac{\lambda_\ell}{\lambda_\ell + \frac{\alphabr}{n} }}
	&= \sqrt{\lambda_i} \parens*{ s^* - \frac{\lambda_i}{\lambda_i + b}} \\
	&= \sqrt{\lambda_j} \parens*{\frac{\lambda_j}{\lambda_j + b} - s^*}.
\end{align*}
Note that the last expression is increasing in $\lambda_j$, so we can take $j = 1$.
Solving for $s^*$ gives
\[
s^* = \frac{\lambda_i \lambda_1 + b(\lambda_i + \lambda_1 - \sqrt{\lambda_i \lambda_1})}{(b + \lambda_i)(b + \lambda_1)}.
\]
Plugging this into the objective function gives
\[
\norm{s^* \scrI_G - \scrSbr}_{\scrH \to L_2}
= \max_i~ \frac{ b \sqrt{\lambda_i \lambda_1}(\sqrt{\lambda_1} - \sqrt{\lambda_i})}{(b + \lambda_1)(b+\lambda_i)}.
\]
One can check that if $\lambda_p \geq \frac{\lambda_1}{\parens*{1 + \sqrt{1 + \frac{\lambda_1}{b}}}^2}$,
this minimum is achieved for $i = p$.
Otherwise,
we can find an upper bound by optimizing over continuous $\lambda$:
\begin{align*}
	\max_i~ \frac{ b \sqrt{\lambda_i \lambda_1}(\sqrt{\lambda_1} - \sqrt{\lambda_i})}{(b + \lambda_1)(b+\lambda_i)}
	&\leq \max_{\lambda \geq 0}~\frac{ b \sqrt{\lambda \lambda_1}(\sqrt{\lambda_1} - \sqrt{\lambda})}{(b + \lambda_1)(b+\lambda)} \\
	&= \frac{b \lambda_1^{3/2}}{2(b + \lambda_1)(b + \sqrt{b(b+\lambda_1)})},
\end{align*}
where the minimum is achieved at $\lambda = \frac{\lambda_1}{\parens*{1 + \sqrt{1 + \frac{\lambda_1}{b}}}^2}$.

Whatever value of $\lambda$ we use, we then have, for the corresponding choice of $s$,
\[
\frac{\norm{s \scrI_G - \scrSbr}_{\scrH \to L_2}}{s} = \frac{b \sqrt{\lambda_1 \lambda}(\sqrt{\lambda_1} - \sqrt{\lambda})}{\lambda_1 \lambda + b(\lambda_1 + \lambda - \sqrt{\lambda_1 \lambda})}.
\]
For $\lambda = \frac{\lambda_1}{\parens*{1 + \sqrt{1 + \frac{\lambda_1}{b}}}^2}$,
we get
\[
\frac{\norm{s \scrI_G - \scrSbr}_{\scrH \to L_2}}{s}
\leq \frac{\sqrt{b \lambda_1(b + \lambda_1)}}{2b + 2\lambda_1 + \sqrt{b(b+\lambda_1)}}.
\]
If $\frac{\alphabr}{n} = b \gtrsim \lambda_1$,
then this last bound is approximately $\sqrt{\lambda_1} \approx \norm{\scrI_G}_{\scrH \to L_2}$,
so there appears to be little hope of getting small classification error from this bound.

Alternatively, if $\frac{\alphabr}{n} \ll \lambda_1$,
we get
\[
\frac{\norm{s \scrI_G - \scrSbr}_{\scrH \to L_2}}{s} \lesssim \sqrt{\frac{\alphabr}{n}}.
\]
However, recall that the \emph{regression} error is of the same order,
so this analysis does not significantly improve our classification risk.

Therefore, the only regime in which we gain anything over the regression analysis is when $\lambda_p > \frac{\lambda_1}{\parens*{1 + \sqrt{1 + \frac{\lambda_1}{b}}}^2}$.

If $b \gtrsim \lambda_1$, then this constraint implies that $\lambda_p/\lambda_1$ is not very small.
Furthermore,
\[
\frac{\norm{s^* \scrI_G - \scrSbr}_{\scrH \to L_2}}{s^*} \approx \sqrt{\frac{\lambda_p}{\lambda_1}}(\sqrt{\lambda_1} - \sqrt{\lambda_p}).
\]
Since $\lambda_p$ is not too small, this ratio is only small when $\lambda_1$ and $\lambda_p$ are very close together.

If $b \lesssim \lambda_1$,
the constraint implies $\lambda_p \gtrsim b$.
Then
\[
\frac{\norm{s^* \scrI_G - \scrSbr}_{\scrH \to L_2}}{s^*} \approx \frac{b}{\sqrt{\lambda_1 \lambda_p}} (\sqrt{\lambda_1} - \sqrt{\lambda_p}).
\]
This is better than the previous case when $b$ is small,
and it improves over the regression error bound when $\lambda_1$ and $\lambda_p$ are close.
However, note that in this case we get $c \gtrsim 1$, so unless $\lambda_p$ is very close to $\lambda_1$,
there is no significant improvement over regression error.

\end{document}